\documentclass{article}

\PassOptionsToPackage{square,sort,comma,numbers}{natbib}

\usepackage[main, final]{neurips_2026}

\usepackage[utf8]{inputenc} 
\usepackage[T1]{fontenc}    

\usepackage[colorlinks=true,citecolor=blue,linkcolor=blue,hypertexnames=false]{hyperref}
\newcommand{\circled}[1]{\text{\raisebox{.6pt}{\textcircled{\raisebox{-.8pt}{\scriptsize #1}}}}}

\usepackage{amssymb}
\usepackage{amsmath,mathrsfs,dsfont}
\usepackage{nicefrac}
\usepackage{algorithm}
\usepackage{algorithmicx}
\usepackage{algpseudocode}

\usepackage{color}
\usepackage[dvipsnames,svgnames,table]{xcolor}
\usepackage{enumitem}

\usepackage{array}
\usepackage{graphicx,tikz}
\usepackage[mathscr]{euscript}
\usepackage{amsthm}
\usepackage{cite}

\usepackage{bm}
\usepackage{bbm}

\usepackage{epstopdf}
\usepackage{cleveref}
\usepackage{thmtools}
\usepackage{thm-restate}

\usepackage{bbding}
\usepackage{mathtools}
\counterwithin{equation}{section}
\allowdisplaybreaks

\newcommand{\tT}{\tilde T}
\newcommand{\dVC}{d^{\textup{(VC)}}}

\usepackage{framed}
\usepackage{xcolor}

\newcommand{\cfrakR}{\mathfrak{R}} 

\newcommand{\Span}{\mathop\mathrm{Span}}

\newcommand{\Ent}{\mathop\mathrm{Ent}}
\newcommand{\barZ}{\overline{Z}}

\graphicspath{{illustrations/}}
\newcounter{optproblem}

\newtheoremstyle{mytheoremstyle} 
    {\topsep}                    
    {\topsep}                    
    {\normalfont}                
    {}                           
    {\bfseries}                   
    {.}                          
    {.5em}                       
    {}  

\theoremstyle{mytheoremstyle}
\newtheorem{theorem}{Theorem}[section]
\newtheorem{remark}[theorem]{Remark}
\newtheorem{proposition}[theorem]{Proposition}
\newtheorem{corollary}[theorem]{Corollary}
\newtheorem{example}[theorem]{Example}

\newtheorem*{theorem*}{Theorem}
\newtheorem*{lemma*}{Lemma}
\newtheorem*{remark*}{Remark}
\newtheorem*{claim*}{Claim}

\newtheorem{lemma}[theorem]{Lemma}

\newtheorem{assumption}{Assumption}

\theoremstyle{remark}
\newtheorem{definition}{Definition}[section]

\DeclareMathAlphabet{\pazocal}{OMS}{zplm}{m}{n}
\DeclareMathAlphabet{\mathpzc}{OMS}{pzc}{m}{it}

\setlist[itemize]{leftmargin=*}

\renewcommand{\hat}{\widehat}

\newcommand{\bfm}[1]{\ensuremath{\mathbf{#1}}}
\newcommand{\bfsym}[1]{\ensuremath{\boldsymbol{#1}}}

\def\ba{\boldsymbol a}   \def\bA{\bfm A}

   \def\bI{\bfm I}  
\def\bj{\bfm j}   \def\bJ{\bfm J}  
   \def\bK{\bfm K}

     \def\NN{\mathbb{N}}

     \def\RR{\mathbb{R}}
   \def\bS{\bfm S}  
     
   \def\bU{\bfm U}

\def\bx{\bfm x}   \def\bX{\bfm X}  
\def\by{\bfm y}   \def\bY{\bfm Y}

 \def\cA{{\cal  A}}

 \def\cE{{\cal  E}}
 \def\cF{{\cal  F}}
 \def\cG{{\cal  G}}
 \def\cH{{\cal  H}}

 \def\cL{{\cal  L}}

 \def\cO{{\cal  O}}

 \def\cS{{\cal  S}}

 \def\cX{{\cal  X}}
 \def\cY{{\cal  Y}}

\def\balpha{\bfsym \alpha}

\def\bsigma{\bfsym \sigma}

\def\+#1{\mathcal{#1}}
\def\-#1{\textup{#1}}

\def\set#1{\left\{ #1 \right\}}
\def\pth#1{\left( #1 \right)}
\def\bth#1{\left[ #1 \right]}
\def\abth#1{\left | #1 \right |}

\def\defeq {\coloneqq}

\DeclareMathSymbol{\relcolon}{\mathrel}{operators}{"3A}

\newcommand{\La}{\left\langle\kern-0.64ex\left\langle}
\newcommand{\Ra}{\right\rangle\kern-0.64ex\right\rangle}

\def\Norm#1#2{{\left\vert\kern-0.4ex\left\vert\kern-0.4ex\left\vert #1
    \right\vert\kern-0.4ex\right\vert\kern-0.4ex\right\vert}_{#2}}
\def\norm#1#2{{\left\|#1\right\|}_{#2}}

\newcommand{\rank}{\textup{rank}}

\def\tr#1{\textup{tr}\left(#1\right)}

\newcommand{\1}{{\rm 1}\kern-0.25em{\rm I}}
\def\indict#1{{\rm 1}\kern-0.25em{\rm I}_{\set{#1}}}

\def \eps  {\epsilon}
\def \eps {\varepsilon}

\def \diff {{\rm d}}
\def \iprod#1#2{\left\langle #1, #2 \right\rangle}

\def\set#1{\left\{#1\right\}}

\def\unitsphere#1{\mathbb{S}^{#1}}
\DeclareMathOperator*{\argmin}{arg\,min}

\def \E {\mathbb{E}}
\def\Expect#1#2{\E_{#1}\left[#2\right]}
\def\Cov  {\textup{cov}}

\def \Pr {\textup{Pr}}
\newcommand{\Prob}[1]{\Pr\left[#1\right]}
\def \Var#1{\textup{Var}\left[#1\right]}

\def \Ent  {\textup{Ent}}

\newcommand{\beq}{\begin{equation}}
\newcommand{\eeq}{\end{equation}}
\newcommand{\beqa}{\begin{eqnarray}}
\newcommand{\eeqa}{\end{eqnarray}}
\newcommand{\beqas}{\begin{eqnarray*}}
\newcommand{\eeqas}{\end{eqnarray*}}
\def\bal#1\eal{\begin{align}#1\end{align}}
\def\bals#1\eals{\begin{align*}#1\end{align*}}
\def\bsal#1\esal{\begin{small}\begin{align}#1\end{align}\end{small}}
\def\bsals#1\esals{\begin{small}\begin{align*}#1\end{align*}\end{small}}
\def\bsfal#1\esfal{\begin{small}\begin{flalign}#1\end{flalign}\end{small}}

\title{Even Sharper Bounds for Transductive Learning and Its Applications}

\author{%
  Yingzhen Yang \\
  School of Computing and Augmented Intelligence\\
  Arizona State University\\
  Tempe, AZ 85281 \\
  \texttt{yingzhen.yang@asu.edu} \\
}

\begin{document}

\maketitle

\begin{abstract}
We introduce Sharper Transductive Local Complexity (STLC) as a new tool for
analyzing the generalization performance of transductive learning methods and
improving current transductive bounds.  Our work extends the classical local
complexity-based analysis to the transductive setting, incorporating
substantial components beyond standard inductive and transductive analyses.
Although Local Rademacher Complexity (LRC) has been used to obtain sharp
inductive generalization bounds and local-complexity-based transductive bounds,
it has remained open whether a localized Rademacher-complexity framework can
achieve the same fixed-point and confidence structure as its inductive
counterpart.  STLC provides an affirmative answer to this question.  STLC is
constructed by first deriving a new Bernstein-type concentration inequality
for the supremum of empirical processes capturing the gap between test and
training losses, or the test--train process, under uniform sampling without
replacement.  The proof uses a new two-parameter entropy argument based on the
modified log-Sobolev inequality for the swap walk.  A subsequent peeling
strategy with a surrogate localization functional yields transductive
excess-risk bounds that match the fixed-point and confidence structure of the
classical LRC-based inductive bounds without the additional logarithmic
confidence gap in existing works.  We further advance transductive learning
through two applications.  First, for realizable transductive learning over a
binary-valued class of VC dimension $2\le\dVC<\infty$ with $u\ge m\ge\dVC$, where $u$
and $m$ are the test and training sizes, STLC gives the upper bound
$\cO\pth{\dVC\log(me/\dVC)/m}$.  This rate matches the classical
LRC-based, standard ERM-style inductive upper bound.  It remains a logarithmic factor
$\log(me/\dVC)$ above both the optimal realizable inductive PAC rate and the
transductive minimax lower bound, which are of order $\dVC/m$.
\citet{yang2026improvedgeneralizationboundstransductive} achieves the excess risk of $\cO(\dVC/m)$ in a transductive setup by
randomized learning.  Second, STLC gives a sharper,
spectrum-adaptive excess-risk bound for transductive kernel learning than the
prior local-complexity-based result.
\end{abstract}

\section{Introduction}
\label{sec:introduction}
We study transductive learning, where the learner has access to labeled
training data and unlabeled test data and aims to predict the labels of the
test data.  Obtaining tight generalization bounds in this setting is a central
problem in statistical learning theory.  Classical tools from inductive
learning, such as Rademacher complexity and VC dimension, have been applied to
transductive empirical risk minimization, regression, and
classification~\citep{Vapnik1982,Vapnik1998,
Cortes2006-transductive-regression,El-Yaniv2009-TRC}.  To obtain sharper
generalization bounds, localized measures such as Local Rademacher Complexity
(LRC)~\citep{bartlett2005} have also been used in transductive
learning~\citep{TolstikhinBK14-local-complexity-TRC}.  The classical LRC
framework gives a sharp excess-risk bound for an empirical risk minimizer
$\hat f$ in inductive learning.  In schematic form, for every $x>0$, with
probability at least $1-\exp(-x)$,
\bal\label{eq:conceptual-risk-bound-inductive}
\textsf{Excess Risk of }\hat f\le\cO\pth{\textsf{Fixed Point for Certain Empirical Process}+\frac xs},
\eal
where $s$ is the inductive training-sample size and the hidden constant is
universal or depends only on fixed model parameters.  Since LRC attains
minimax rates for many inductive $M$-estimators, including nonparametric
regression, a central question for transductive learning under sampling
without replacement is whether a local-complexity bound can have exactly the
same fixed-point and confidence structure as the inductive bound
(\ref{eq:conceptual-risk-bound-inductive}).

One of the most relevant earlier results is
\citep[Theorem~3.6]{yang2025a-concentration-sampling-without-replacement},
which overcomes the difficulties caused by independence-based analysis and
gives the schematic bound
\bal\label{eq:conceptual-risk-bound-transductive-imbalanced}
\textsf{Excess Risk of }\hat f
\le\cO\pth{r_{u,m}+r^*+\frac{x}{\min\set{u,m}}}.
\eal
Here $u$ and $m$ are the test and training sizes, and $r_{u,m}$ and $r^*$ are
fixed points of localized empirical-process majorants that converge rapidly
in standard learning models.  However,
(\ref{eq:conceptual-risk-bound-transductive-imbalanced}) is proved only in the
highly imbalanced regimes $u\gg m^2$ or $m\gg u^2$, which substantially
limits its applicability.

To remove this restriction,
\citep[Theorem~4.11]{yang2026improvedgeneralizationboundstransductive} proves
that, for every $\delta\in(0,1/3)$, with probability at least
$1-3\exp(-x)-3\delta$,
\bal\label{eq:conceptual-risk-bound-transductive-log-gap}
\textsf{Excess Risk of }\hat f
\le\cO\pth{r_{u,m}+r^*
+\frac{\log_2(4\min\set{u,m}/\delta)x}{\min\set{u,m}}}.
\eal
Nevertheless, (\ref{eq:conceptual-risk-bound-transductive-log-gap}) retains
the logarithmic factor $\log_2(4\min\set{u,m}/\delta)$ relative to the sharp
inductive bound (\ref{eq:conceptual-risk-bound-inductive}).  This gap is
substantial when $\delta$ is exponentially small.  Existing transductive
results also suggest that, as the test size $u$ tends to infinity, a
transductive excess-risk bound should approach its inductive counterpart
\citep{tolstikhin2016minimaxlowerboundsrealizable}.

One of our main results is the following improved transductive excess-risk
bound.  For every $x>0$, with probability at least $1-3\exp(-x)$,
\bal\label{eq:conceptual-risk-bound-STLC}
\textsf{Excess Risk of }\hat f
\le\cO\pth{r_{u,m}+r^*+\frac{x}{\min\set{u,m}}}.
\eal
The formal statement is
Theorem~\ref{theorem:STLC-delta-ell-f-excess-risk-upper-bound}.  Equation
(\ref{eq:conceptual-risk-bound-STLC}) answers the question above affirmatively
by removing both the imbalance requirement in
(\ref{eq:conceptual-risk-bound-transductive-imbalanced}) and the logarithmic
confidence gap in (\ref{eq:conceptual-risk-bound-transductive-log-gap}).

As two key applications, we apply STLC to realizable transductive learning over
binary-valued classes of finite VC dimension and to transductive kernel
learning, summarized in Section~\ref{sec:summary-main-results} and detailed in
Section~\ref{sec:applications}.  The concentration inequality for the
test--train process in Theorem~\ref{theorem:concentration-gd-STLC} uses a new
two-parameter entropy argument based on the modified log-Sobolev inequality for
the swap walk on the Johnson graph
\citep{Cryan2021-modified-log-sobolev-inequality}.  For a realizable
$\set{0,1}$-valued class with VC dimension $2\le\dVC<\infty$ and
$u\ge m\ge\dVC$, Theorem~\ref{theorem:STLC-delta-ell-f-excess-risk-upper-bound-VC-dim}
gives the nearly minimax-optimal, up to the factor $\log(me/\dVC)$, upper bound
$\cO\pth{\dVC\log(me/\dVC)/m}$.  For $m\ge9$, this is within a logarithmic
factor of the minimax lower bound $(\dVC-1)/(16m)$ in
\citep[Theorem~3]{tolstikhin2016minimaxlowerboundsrealizable} and matches the
classical LRC-based inductive upper bound, not the optimal realizable
inductive PAC rate of order $\dVC/m$.
\citet{yang2026improvedgeneralizationboundstransductive} distinguishes this uniform
ERM bound from log-free randomized learning in a transductive setup, which also achieves the excess risk of $\cO(\dVC/m)$.
Theorem~\ref{theorem:STLC-kernel} gives a
spectrum-adaptive excess-risk bound for transductive kernel learning without
the multiplicative imbalance factors in the earlier local-complexity result
\citep{TolstikhinBK14-local-complexity-TRC}.  Detailed comparisons with related
work appear in Appendix~\ref{sec:more-comparison}.

\subsection{Notations}
We use bold letters for matrices and vectors and regular lowercase letters for
scalars.  A bold letter with a parenthesized superscript denotes a matrix
column, so $\bA^{(i)}$ is the $i$th column of $\bA$.  A bold letter with a
subscript denotes a matrix or vector entry.  For a vector $\ba$, $\ba(i)$ is
its $i$th entry, $\ba(i\relcolon j)$ is the vector formed by entries with
indices from $i$ through $j$, and $\ba(C)$ is the vector formed by entries
whose indices lie in $C$.  The symbols $\norm{\cdot}{F}$ and
$\norm{\cdot}{p}$ denote the Frobenius norm and the vector $\ell_p$-norm or
matrix $p$-norm.  The symbols $\Var{\cdot}$, $\bI_n$, and $\indict{E}$ denote
variance, the $n\times n$ identity matrix, and the indicator of an event $E$.
For a finite set $A$, $\overline A$ is its complement in the ambient finite
set and $\abth{A}$ is its cardinality.  The symbol $\tr{\cdot}$ denotes matrix
trace, and $\unitsphere{d-1}\defeq\set{\bx\in\RR^d\colon
\norm{\bx}{2}=1}$ is the unit sphere in $\RR^d$.

For a probability measure $\mu$ on $\cX$, let $L^2(\cX,\mu)$ denote the
space of square-integrable functions.  Its inner product and norm are
$\iprod{f}{g}_{L^2(\mu)}\defeq\int_{\cX}f(x)g(x)\diff\mu(x)$ and
$\norm{f}{L^2(\mu)}^2\defeq\int_{\cX}f^2(x)\diff\mu(x)$.  Inner products and
norms in a Hilbert space $\cH$ are denoted by
$\iprod{\cdot}{\cdot}_{\cH}$ and $\norm{\cdot}{\cH}$.

We write $a\lesssim b$ or $a=\cO(b)$ when $a\le Cb$ for a constant $C>0$
whose permitted dependence is stated locally.  The notation
$a=\widetilde{\cO}(b)$ indicates additional requirements on the hidden
constant that are stated where the notation is used.  We write $a\asymp b$ or
$a=\Theta(b)$ when $cb\le a\le Cb$ for constants $c,C>0$.  The notation
$a=o(b)$ and $a=\omega(b)$ has its usual asymptotic meaning.  We set
$\RR^+=[0,\infty)$ and $\NN=\set{1,2,\ldots}$.  For integers $p,q$, let
$[p\relcolon q]=\set{p,p+1,\ldots,q}$ when $p\le q$ and
$[p\relcolon q]=\emptyset$ when $p>q$.  We abbreviate
$[1\relcolon q]$ as $[q]$.  An empty finite sum is zero.
\vspace{-4pt}
\section{Problem Setup of Transductive Learning}
\label{sec:transductive-basic-setup}
\vspace{-4pt}
We consider the indexed full sample $((\bx_i,y_i))_{i=1}^{n}$, where
$n=m+u$, $\bx_i\in\cX\subseteq\RR^d$, and $y_i\in\cY\subseteq\RR$.
Repeated feature values are allowed because sampling is over indices.
Put $N_{u,m}\defeq\min\set{u,m}$ and assume $N_{u,m}\ge2$.
The learner is provided with the unlabeled indexed collection
$\bX_n\defeq(\bx_i)_{i=1}^n$.  Under the standard transductive protocol
\citep{El-Yaniv2009-TRC,TolstikhinBK14-local-complexity-TRC}, a training
index set $\barZ$ of size $m$ is sampled uniformly without replacement from
$[n]$.  Put $Z\defeq[n]\setminus\barZ$ and define
$\bX_m\defeq(\bx_i)_{i\in\barZ}$ and
$\bX_u\defeq(\bx_i)_{i\in Z}$.
It follows by symmetry that the test index set $Z$ is sampled uniformly from all subsets of $[n]$ of size $u$.
For a function $g$ on $\cX\times\cY$, or on $\cX$, we identify its
restriction to the full sample with the function on $[n]$ given by
$g(i)=g(\bx_i,y_i)$, or $g(i)=g(\bx_i)$.  Given a predictor
$f\colon\cX\to\RR$ and a loss $\ell\colon\RR\times\cY\to\RR$, define
$\ell_f(i)\defeq\ell(f(\bx_i),y_i)$.
For any real-valued function $h$ on $[n]$ and any nonempty
$\cA\subseteq[n]$, define
$\cL_h(\cA)\defeq \abth{\cA}^{-1}\sum_{i\in\cA}h(i)$.
We abbreviate the full-sample, training, and test averages
$\cL_h([n])$, $\cL_h(\bar Z)$, and $\cL_h(Z)$ as
$\cL_n(h)$, $\cL_m(h)$, and $\cL_u(h)$, respectively.  We write
$T_n(h) \defeq \sum_{i=1}^n h^2(i)/n=\cL_n(h^2)$
for the full-sample second moment of $h$.  This notation is used throughout,
including in the localization arguments and the kernel-learning application.
When $h=\ell_f$, the quantities $\cL_m(h)$ and $\cL_u(h)$ are the training
and test losses of $f$.
Unless stated otherwise, probabilities and expectations involving the random
split are taken over $Z$.  In particular,
$\Expect{}{\cL_m(h)}=\Expect{}{\cL_u(h)}=\cL_n(h)$ for every fixed $h$.

The labeled training collection is
$\bS_m\defeq((\bx_i,y_i))_{i\in\barZ}$.  The learner uses
$\bS_m$ and $\bX_u$ to predict the labels indexed by $Z$.

Table~\ref{tab:inductive-transductive} records the distinction between the
inductive and transductive protocols.  In particular, the two settings may
share the form of a rate without sharing the same probability model.
\begin{table}[H]
\caption{Inductive and transductive learning.  The transductive split is a
random partition of a fixed indexed full sample.}
\label{tab:inductive-transductive}
\centering
\small
\setlength{\tabcolsep}{4pt}
\renewcommand{\arraystretch}{1.08}
\begin{tabular}{@{}>{\raggedright\arraybackslash}p{0.18\linewidth}>{\raggedright\arraybackslash}p{0.35\linewidth}>{\raggedright\arraybackslash}p{0.39\linewidth}@{}}
\hline
Feature & Inductive learning & Transductive learning \\
\hline
Data & Labeled observations drawn independently from an unknown law & Fixed full feature collection $\bX_n$ known before the labels indexed by $Z$ are predicted \\
Randomness & Independent training observations & Uniform size-$m$ training subset whose complement is the test set \\
Objective & Population risk on future observations & Average loss on the observed test features \\
Main obstacle & Sampling error relative to a population & Dependence between complementary training and test averages \\
Available geometry & Training features & Full-sample geometry, including the Gram spectrum \\
\hline
\end{tabular}
\end{table}
The complementarity of $Z$ and $\bar Z$ invalidates direct independent
resampling arguments.  The swap walk preserves both subset sizes, and its
modified log-Sobolev inequality therefore supplies the geometry-adapted
concentration mechanism used below.  Conversely, transduction can exploit
the known test features, as illustrated by the full-sample kernel spectrum in
Theorem~\ref{theorem:STLC-kernel}.

\vspace{-4pt}
\section{Summary of Main Results}
\label{sec:summary-main-results}
\vspace{-4pt}
We summarize our main results in this section, with the basic definitions introduced first.

\subsection{Basic Definitions}
For a nonempty class $\cH$ of real-valued functions on $[n]$ and a nonempty
proper subset $\cA\subsetneq[n]$, define the test--train supremum
\bal\label{eq:def-g}
g(\cH,\cA) \defeq \sup_{h \in \cH} \pth{\cL_h(\cA)- \cL_h\pth{[n]\setminus\cA} }.
\eal
Thus $g(\cH,Z)$ is the test--train process, while $g(\cH,\bar Z)$ exchanges
the roles of the test and training sets.
We next introduce Rademacher variables and transductive complexity.
\begin{definition}[Rademacher variables]
\label{def:rad-variables}
A Rademacher variable is a random variable $\sigma$ satisfying
$\Prob{\sigma=1}=\Prob{\sigma=-1}=1/2$.  Unless stated otherwise,
$\sigma_1,\ldots,\sigma_n$ denote independent Rademacher variables.
\end{definition}
For comparison with the complexity used in the proofs, let
$\bY^{(p)}=(Y_1,\ldots,Y_p)$ have independent coordinates uniformly
distributed on $[n]$, independently of $\bsigma=(\sigma_1,\ldots,\sigma_p)$,
for $p\in\set{u,m}$.  The corresponding inductive, with-replacement
Rademacher complexity is
\bals
\cfrakR_p^{(\textup{ind})}(\cH)
\defeq\Expect{\bY^{(p)},\bsigma}{\sup_{h\in\cH}
\frac1p\sum_{i=1}^p\sigma_i h(Y_i)}.
\eals
The El-Yaniv--Pechyony transductive Rademacher complexity
\citep{El-Yaniv2009-TRC} instead uses independent ternary multipliers on all
$n$ fixed indices.  Permutational Rademacher complexity uses a random balanced
partition of fixed indices.  By contrast, the TC below is the expected
one-sided deviation of the actual uniformly sampled subset from the fixed
full-sample mean.  It uses neither independent ternary multipliers nor an
auxiliary balanced partition.  STLC localizes this TC through the surrogate
functional introduced in Section~\ref{sec:STLC-bound-generic}.
The following definition recalls transductive complexity from
\citep{yang2025a-concentration-sampling-without-replacement}.
\begin{definition}[Transductive Complexity, or TC, {\citep[Definition 2.2]{yang2025a-concentration-sampling-without-replacement}}]
\label{def:TC}
Define $R^+_{u} h \defeq \cL_h(Z) - \cL_n(h)$, $R^+_{m} h \defeq \cL_h(\barZ) - \cL_n(h)$,
and $R^{-}_{u} h \defeq -R^+_{u} h$,
$R^-_{m} h \defeq -R^+_{m} h$.
The four types of Transductive Complexity (TC) of a function class $\cH$ are defined as
\bal\label{eq:TC-def}
\cfrakR^+_u(\cH)\defeq\Expect{}{\sup_{h\in\cH}R^+_uh},\qquad
\cfrakR^-_u(\cH)\defeq\Expect{}{\sup_{h\in\cH}R^-_uh},\nonumber\\
\cfrakR^+_m(\cH)\defeq\Expect{}{\sup_{h\in\cH}R^+_mh},\qquad
\cfrakR^-_m(\cH)\defeq\Expect{}{\sup_{h\in\cH}R^-_mh}.
\eal
\end{definition}
The expectations in (\ref{eq:TC-def}) are over $Z$ where the subset $Z$ is sampled uniformly from $[n]$ without
replacement.
Sub-root functions encode the localized complexity bounds used below.
\begin{definition}[Sub-root function, {\citep[Definition 3.1]{bartlett2005}}]
\label{def:sub-root-function}
A function $\psi \colon [0,\infty) \to [0,\infty)$ is sub-root if it is
nonnegative and nondecreasing, and if $r\mapsto\psi(r)/\sqrt r$ is
nonincreasing on $(0,\infty)$.  A positive number $r^*$ is a fixed point of
$\psi$ if $\psi(r^*)=r^*$.
\end{definition}
Every nonzero sub-root function has a unique positive fixed point $r^*$, and
for $r>0$, $\psi(r)\le r$ holds exactly when $r\ge r^*$
\citep[Lemma 3.2]{bartlett2005}.  We use this fixed-point comparison below.

\vspace{-7pt}
\subsection{Main Results}
The results have three layers.  First,
Theorem~\ref{theorem:concentration-gd-STLC} proves a Bernstein-type tail bound
for the test--train supremum by exploiting the swap-walk geometry.  Second,
Theorem~\ref{theorem:TLC} combines that inequality with peeling to obtain a
uniform local-complexity bound.  Third,
Theorem~\ref{theorem:STLC-delta-ell-f-excess-risk-upper-bound} applies the
bound to empirical risk minimization through a surrogate localization
functional.  The resulting confidence term has order
$x/N_{u,m}$ and requires no imbalance condition.

The VC-class specialization in
Theorem~\ref{theorem:STLC-delta-ell-f-excess-risk-upper-bound-VC-dim}
gives $\cO\pth{\dVC\log(me/\dVC)/m}$ for a realizable binary class when
$u\ge m\ge\dVC$.
Theorem~\ref{theorem:STLC-kernel} bounds the generic fixed points by the
empirical kernel spectrum.  Each theorem below is followed by a technical
roadmap, and the complete proofs appear in Appendix~\ref{sec:proofs}.

Table~\ref{tab:main-comparison} consolidates the rate-level comparison.  In
the rows involving $\delta$, $\delta\in(0,1/3)$ is an auxiliary failure
probability.  The phrase ``spectral term'' denotes the leading expression in
Theorem~\ref{theorem:STLC-kernel}.  Within each of the three blocks, the
compared results use the same corresponding assumptions, as detailed in
Table~\ref{tab:assumption-provenance}.
\begin{table}[H]
\caption{Comparison with the two preceding local-complexity analyses.  The
bounds suppress constants depending only on fixed model parameters.}
\label{tab:main-comparison}
\centering
\scriptsize
\setlength{\tabcolsep}{3pt}
\renewcommand{\arraystretch}{1.08}
\begin{tabular}{@{}>{\raggedright\arraybackslash}p{0.30\linewidth}>{\raggedright\arraybackslash}p{0.28\linewidth}>{\raggedright\arraybackslash}p{0.35\linewidth}@{}}
\hline
Result & Regime and probability & Bound \\
\hline
Generic, \citep[Theorem~3.6]{yang2025a-concentration-sampling-without-replacement}
& $u\gg m^2$ or $m\gg u^2$, $1-3\exp(-x)-3N_{u,m}^2/\max\set{u,m}$
& $r_{u,m}+r^*+x/N_{u,m}$ \\
Generic, \citep[Theorem~4.11]{yang2026improvedgeneralizationboundstransductive}
& arbitrary $u,m$, $1-3\exp(-x)-3\delta$
& $r_{u,m}+r^*+\log_2(4N_{u,m}/\delta)x/N_{u,m}$ \\
Generic, Theorem~\ref{theorem:STLC-delta-ell-f-excess-risk-upper-bound}
& arbitrary $u,m$, $1-3\exp(-x)$
& $r_{u,m}+r^*+x/N_{u,m}$ \\
\hline
Realizable VC, \citep[Theorem~6.1]{yang2026improvedgeneralizationboundstransductive}
& $u\ge m\ge\dVC$, $1-\exp(-x)-\delta$
& $\dVC\log(me/\dVC)/m+\log_2(4m/\delta)x/m$ \\
Realizable VC, Theorem~\ref{theorem:STLC-delta-ell-f-excess-risk-upper-bound-VC-dim}
& $u\ge m\ge\dVC$, $1-\exp(-x)$
& $\dVC\log(me/\dVC)/m+x/m$ \\
\hline
Kernel, \citep[Theorem~4.1]{yang2025a-concentration-sampling-without-replacement}
& $u\gg m^2$ or $m\gg u^2$, $1-3\exp(-x)-3N_{u,m}^2/\max\set{u,m}$
& spectral term $+x/N_{u,m}$ \\
Kernel, \citep[Theorem~6.2]{yang2026improvedgeneralizationboundstransductive}
& arbitrary $u,m$, $1-3\exp(-x)-3\delta$
& spectral term $+\log_2(4N_{u,m}/\delta)x/N_{u,m}$ \\
Kernel, Theorem~\ref{theorem:STLC-kernel}
& arbitrary $u,m$, $1-3\exp(-x)$
& spectral term $+x/N_{u,m}$ \\
\hline
\end{tabular}
\end{table}
Thus the earlier work already replaces the leading VC factor $\log n$ by
$\log m$.  The VC improvement here is specifically the removal of its
auxiliary failure probability and logarithmic confidence penalty.  In the
generic and kernel settings, STLC additionally removes the earlier imbalance
restriction or, relative to the balanced-regime result, the same confidence
penalty.

\vspace{-7pt}
\section{STLC-Based Excess Risk Bounds for Generic Transductive Learning}
\label{sec:detailed-results}
\vspace{-3pt}
The argument proceeds from concentration to localization and then to excess
risk.  Theorem~\ref{theorem:concentration-gd-STLC} controls the supremum of
the differences between test and training averages.
Theorem~\ref{theorem:TLC} localizes this bound through a surrogate for the
second moment.  Theorem~\ref{theorem:STLC-delta-ell-f-excess-risk-upper-bound}
then applies it to the loss difference between the training minimizer and
the test oracle, using their full-sample excess losses to control the
surrogate.  The confidence term is $x/N_{u,m}$ for every split size.

\noindent\colorbox{gray!8}{\parbox{\dimexpr\linewidth-2\fboxsep\relax}{\small
\textbf{Guide to the radius notation.}  The symbol $r$ is a variable
localization radius.  The quantity $r_{u,m}$ is the positive fixed point of
the sub-root majorant $\psi_{u,m}$ in
(\ref{eq:STLC-cond-psi-general}).  The quantity $r^*$ is the positive fixed
point of the excess-loss majorant $\psi^*_{u,m}$ in
(\ref{eq:STLC-cond-um-delta-star-ell-f-psi-star}).  The later quantity
$r(u,m,Q)$ is a kernel spectral expression indexed by $Q$.
}}
\subsection{Concentration Inequality for the Test-Train Process}
\label{sec:starting-concentration}
To control fluctuations of $g(\cH,Z)$, we use second moments for the
variance scale and the complexity of the squared functions for the random
variance created by the supremum.  Write
$\cH^2\defeq\set{h^2\colon h\in\cH}$ and recall
$N_{u,m}=\min\set{u,m}$.
\begin{theorem}[Concentration of the test and training supremum]
\label{theorem:concentration-gd-STLC}
Fix integers $u,m\ge2$, put $n=u+m$, and let $Z$ be uniformly distributed
over the $u$-element subsets of $[n]$.  Write $\bar Z=[n]\setminus Z$.
Let $\cH$ be a nonempty class of real-valued functions on $[n]$, fixed
before the split is sampled.  Suppose that, for some $H_0,r>0$,
$\sup_{h\in\cH}\max_{i\in[n]}\abth{h(i)}\le H_0$ and
$\sup_{h\in\cH}T_n(h)\le r$.
Set $\cfrakR^+_{N_{u,m}}=\cfrakR^+_u$ when $u\le m$ and
$\cfrakR^+_{N_{u,m}}=\cfrakR^+_m$ when $m<u$.  When $u=m$, these two
quantities agree by complementation.
Set $c_0\defeq\sqrt{219}$.
Then, for every $x>0$, with probability at least $1-\exp(-x)$ over the
uniformly sampled set $Z$,
\bal\label{eq:concentration-gd-STLC}
g(\cH,Z)\le \Expect{}{g(\cH,Z)}
+c_0\sqrt{\frac{rx}{N_{u,m}}}
+\frac{c_0}{2}\cfrakR^+_{N_{u,m}}(\cH^2)
+\frac{(16H_0+c_0/2)x}{N_{u,m}}.
\eal
The transductive complexities are as in Definition~\ref{def:TC}.  Since the
sampling space is finite and the class is
uniformly bounded, all suprema in the statement are finite measurable random
variables.  No attainment assumption on $\cH$ is required.
\end{theorem}
Corollary~\ref{corollary:concentration-gu} in
Appendix~\ref{sec:concentration-general-sup-empirical} deduces simultaneous
concentration for the two one-sided processes centered at $\cL_n$.  That
appendix also compares the result with earlier sampling-without-replacement
inequalities.

\noindent \textbf{Proof roadmap for Theorem~\ref{theorem:concentration-gd-STLC}.}
For $u\le m$, the proof in Appendix~\ref{sec:proof-concentration-STLC}
uses the swap walk, which exchanges one index of $Z$ with one of $\bar Z$
and preserves the sampling law.
\vspace{-8pt}
\begin{itemize}[leftmargin=0pt, itemsep=0pt]
\item[] \textbf{Control the random variance.}
Lemma~\ref{lemma:oriented-swap-variance-bounds} bounds the transition averages of squared positive swap increments of
both $g(\cH,Z)$ and $Q(Z)=\sup_{h\in\cH}\set{\cL_u(h^2)-T_n(h)}$
by $r$ and $Q(Z)$.
\item[] \textbf{Close the entropy inequality.}
Apply the modified log-Sobolev inequality
(\ref{eq:mlsi-swap}) to a nonnegative linear combination of the two
centered processes.  The resulting differential inequality for their joint
logarithmic moment generating function represents the random variance by
a derivative in the second parameter.
\item[] \textbf{Obtain the tail bound.}
Lemma~\ref{lemma:two-parameter-characteristic-curve} supplies a curve whose
derivative matches the differential inequality, and
Proposition~\ref{proposition:two-parameter-herbst-closure} bounds the
remaining integral.  Chernoff's inequality yields
(\ref{eq:concentration-gd-STLC}).  For $m<u$, complementation and sign
reversal give the same argument on the smaller subset.
\end{itemize}
Thus the square-class complexity in the theorem is the expectation of the
auxiliary variance process, and the swap argument treats the sampling
dependence directly.

\subsection{Comparison with earlier bounds}
\label{sec:comparison-earlier-bounds}

To compare Theorem~\ref{theorem:concentration-gd-STLC} with earlier bounds,
center each function by $h^\circ=h-\cL_n(h)$ and put
$\cH^\circ=\set{h^\circ\colon h\in\cH}$,
$H_\circ=\sup_{h\in\cH}\max_{i\in[n]}\abth{h^\circ(i)}$, and
$\sigma^2=\sup_{h\in\cH}T_n(h^\circ)$.
Write $g_u^+(\cH,Z)=\sup_{h\in\cH}\set{\cL_u(h)-\cL_n(h)}$.
Corollary~\ref{corollary:concentration-gu} gives the centered bound
(\ref{eq:concentration-gu-centered}) in
Appendix~\ref{sec:concentration-general-sup-empirical}.
For comparison with \citet{TolstikhinBK14-local-complexity-TRC}, assume
$H_\circ\le1$ and that the class is countable or separable, as in that work.
Let $\bY^{(u)}=(Y_1,\ldots,Y_u)$ have independent coordinates uniformly
distributed on $[n]$, and define
$\widetilde g_u(\bY^{(u)},\cH)\defeq
\sup_{h\in\cH}u^{-1}\sum_{i=1}^u h^\circ(Y_i)$.
Theorem~1 of that work, after converting
its unnormalized sums to empirical averages, gives, with probability at
least $1-\exp(-t)$ for every $t>0$,
\bal\label{eq:existing-concentration-sup-empirical-process-1}
g_u^+(\cH,Z)-\Expect{}{g_u^+(\cH,Z)}
&\le
\frac{2\sqrt2\,n\sigma}{u\sqrt{n+2}}\sqrt t
\le \frac{2\sigma\sqrt{2nt}}{u}.
\eal
Theorem~2 of the same work gives, with the same probability,
\bal\label{eq:existing-concentration-sup-empirical-process-2}
g_u^+(\cH,Z)-
\Expect{\bY^{(u)}}{\widetilde g_u(\bY^{(u)},\cH)}
\le
\sqrt{\frac{2\set{\sigma^2+2
\Expect{\bY^{(u)}}{\widetilde g_u(\bY^{(u)},\cH)}}t}{u}}
+\frac{t}{3u}.
\eal
To express the second bound around $\Expect{}{g_u^+(\cH,Z)}$, one must
control the difference of the two expectations.  Hoeffding's comparison
principle gives its nonnegativity.  Lemma~3 of
\citet{TolstikhinBK14-local-complexity-TRC} gives the upper bound $2u^2/n$,
but a coupling of the two sampling schemes gives the sharper estimate
\bal\label{eq:with-without-expectation-gap}
0\le
\Expect{\bY^{(u)}}{\widetilde g_u(\bY^{(u)},\cH)}
-\Expect{}{g_u^+(\cH,Z)}
\le\frac{u-1}{n}.
\eal
Appendix~\ref{sec:expectation-gap-comparison} proves
(\ref{eq:with-without-expectation-gap}).  Combining it with
(\ref{eq:existing-concentration-sup-empirical-process-2}) gives
\bal\label{eq:existing-concentration-sup-empirical-process-2-repeat}
g_u^+(\cH,Z)-\Expect{}{g_u^+(\cH,Z)}
&\le
\sqrt{\frac{2\set{\sigma^2+2
\Expect{\bY^{(u)}}{\widetilde g_u(\bY^{(u)},\cH)}}t}{u}}
+\frac{t}{3u}+\frac{u-1}{n}.
\eal
For fixed $t$ and $\sigma$ bounded away from zero, the bound in
(\ref{eq:existing-concentration-sup-empirical-process-1}) diverges when
$u=o(\sqrt n)$.  However, since $\sigma^2\le1$ and
$\Expect{\bY^{(u)}}{\widetilde g_u(\bY^{(u)},\cH)}\le1$, the bound in
(\ref{eq:existing-concentration-sup-empirical-process-2-repeat}) vanishes
whenever $u\to\infty$ and $u=o(n)$, including $u\asymp\sqrt n$.
Thus failure to vanish at that scale is a limitation of the estimate
$2u^2/n$, rather than of the earlier concentration theorem.

Our bound (\ref{eq:concentration-gu-centered}) is centered directly at the
expectation under sampling without replacement.  It uses the localized
square-class complexity and has no additive expectation-gap term.
Theorem~\ref{theorem:TC-RC} and the contraction inequality in
Theorem~\ref{theorem:contraction-RC}, both stated in
Appendix~\ref{sec::math-tools}, give
$\cfrakR_p^+\pth{(\cH^\circ)^2}
\le2\cfrakR_p^{(\textup{ind})}\pth{(\cH^\circ)^2}
\le4H_\circ\cfrakR_p^{(\textup{ind})}(\cH^\circ)$ for
$p\in\set{u,m}$.
Hence, for fixed $x$, bounded $H_\circ$ and $\sigma$, and any model for
which
$\cfrakR_p^{(\textup{ind})}(\cH^\circ)=\cO(p^{-1/2})$, the right-hand side
of (\ref{eq:concentration-gu-centered}) is
$\cO(N_{u,m}^{-1/2})$ and converges to zero whenever
$N_{u,m}\to\infty$, irrespective of the relative growth of $u$ and $m$.
Such Rademacher-complexity behavior holds for many standard bounded classes.
See, for example, \citep{Bartlett2003}.

Finally, the earlier concentration result
\citep[Theorem~5.1]{yang2025a-concentration-sampling-without-replacement}
requires one of the imbalanced conditions $m\gg u^2$ or $u\gg m^2$, while
\citep[Corollary~4.3]{yang2026improvedgeneralizationboundstransductive}
introduces an auxiliary failure parameter and a corresponding logarithmic
factor.  Corollary~\ref{corollary:concentration-gu}, being a direct
consequence of Theorem~\ref{theorem:concentration-gd-STLC}, requires neither
device and gives both one-sided suprema simultaneously at failure probability
$\exp(-x)$.

\subsection{Sharper Transductive Local-Complexity Bound}
\label{sec:first-STLC-bound}
We next obtain one event controlling every $h\in\cH$ by localizing with
a functional $\tT_n\ge T_n$.  This ensures the second-moment bound needed
for concentration on each localized class.
For loss differences, the surrogate in (\ref{eq:tTn-def}) is controlled by
two full-sample excess losses, which connects localization to empirical
minimization.  No homogeneity of $\tT_n$ is required.
\begin{theorem}[Sharper transductive local-complexity bound]\label{theorem:TLC}
Use the sampling setup of Theorem~\ref{theorem:concentration-gd-STLC}.
Let $\cH$ be a nonempty class of real-valued functions on $[n]$, and let
$H_0>0$ satisfy $\abth{h(i)}\le H_0$ for every $h\in\cH$ and $i\in[n]$.
Fix $K_0>1$ and a functional $\tT_n\colon\cH\to\RR^+$.
Assume that $\cH$ and $\tT_n$ are fixed before sampling $Z$ and that
$T_n(h)\le\tT_n(h)$ for every $h\in\cH$.  For $r>0$, put
$\cH(r)\defeq\set{h\in\cH\colon\tT_n(h)\le r}\cup\set{0}$ and
$\cH^2(r)\defeq\set{h^2\colon h\in\cH,\ \tT_n(h)\le r}\cup\set{0}$,
where $0$ denotes the zero function.  Let $\psi_{u,m}$ be a deterministic
sub-root function with positive fixed point $r_{u,m}$, and suppose that, for every
$r\ge r_{u,m}$,
\bal\label{eq:STLC-cond-psi-general}
\psi_{u,m}(r) \ge
\max_{\substack{p\in\set{u,m}\\ \eta\in\set{+,-}}}
\max\set{\cfrakR_p^\eta\pth{\cH(r)},
\cfrakR_p^\eta\pth{\cH^2(r)}}.
\eal
With $c_0=\sqrt{219}$ as in
Theorem~\ref{theorem:concentration-gd-STLC}, define
$d_0\defeq4+4c_0/7$, $c_1\defeq8K_0d_0^2$, and
$c_2\defeq8K_0c_0^2+32H_0+c_0
=1752K_0+32H_0+\sqrt{219}$.
Then, for every $x>0$, with probability at least $1-\exp(-x)$ over $Z$,
\bal\label{eq:STLC-bound-g-upper-bound}
\cL_h(Z) \le \cL_h(\barZ) + \frac{\tT_n(h)}{K_0}
+c_1 r_{u,m} + \frac{c_2x}{N_{u,m}},\qquad h\in\cH.
\eal
\end{theorem}
\noindent\textbf{Proof roadmap for Theorem~\ref{theorem:TLC}.}
Appendix~\ref{sec:proofs-theorem-STLC-STLC-nonnegative-func-class} rescales
each function using a geometric upper approximation to $\tT_n(h)$, giving
one class with second moments at most $r$.
Theorem~\ref{theorem:concentration-gd-STLC} applies once to this class.
Peeling bounds only the expectations, so it adds no failure probability.
The sub-root property sums the shell complexities, giving terms of order
$\sqrt{r r_{u,m}}$ and $\sqrt{rx/N_{u,m}}$.
A deterministic choice of $r$ absorbs these terms, and
Lemma~\ref{lemma:STLC-supp-lemma} returns to the original functions.

\subsection{Sharp Excess Risk Bounds using STLC for Generic Transductive Learning}
\label{sec:STLC-bound-generic}

We now apply Theorem~\ref{theorem:TLC} to the transductive learning task in
Section~\ref{sec:transductive-basic-setup}.  Let $\cF$ be a nonempty class of
functions from $\cX$ to $\RR$.  Throughout this subsection, assume that
$0\le\ell_f(i)\le L_0$ for every $f\in\cF$ and $i\in[n]$, where $L_0>0$.
Whenever the corresponding minima are attained, choose
$\hat f_{u}\in\argmin_{f \in \cF} \cL_u(\ell_f)$
as a test-risk oracle and choose
$\hat f_{m}\in\argmin_{f \in \cF}\cL_m(\ell_f)$
as an empirical minimizer.  Both may be nonunique, and every result holds for
any measurable selection.  Their dependence on the random split $Z$ is
suppressed in the notation.  For $f\in\cF$, define the transductive excess risk
by $\cE(f)\defeq\cL_u(\ell_f)-\cL_u(\ell_{\hat f_u})$.
If the test minimum is not attained, an $\eps$-optimal oracle gives a quantity
within $\eps$ of
$\cL_u(\ell_f)-\inf_{g\in\cF}\cL_u(\ell_g)$.  The theorems below state the
existence assumptions needed for the exact definition and study
$\cE(\hat f_m)$.

Define the loss-difference class
$\Delta_{\cF}\defeq\set{\ell_{f_1}-\ell_{f_2}\colon f_1,f_2\in\cF}$.
We compare each predictor with a full-sample minimizer.
The following Bernstein condition bounds the second moment of an excess
loss by its nonnegative full-sample mean, allowing absorption of the
localization penalty.
\begin{assumption}
[Main Assumption]
\label{assumption:main}
\begin{itemize}[leftmargin=16pt]
\item[(1)] There is a function $f_n^* \in \cF$ such that
$\cL_n(\ell_{f_n^*}) = \inf_{f \in \cF} \cL_n(\ell_f)$.

\item[(2)] Put
$\Delta^*_{\cF}\defeq\set{\ell_f-\ell_{f_n^*}\colon f\in\cF}$.
There is a constant $B>0$ such that
$T_n(h)\le B\cL_n(h)$ for every $h\in\Delta^*_{\cF}$.
\end{itemize}
\end{assumption}

\begin{remark}
Assumption~\ref{assumption:main}(2) is the usual empirical Bernstein
condition: it controls the full-sample second moment of each excess-loss
function by its (nonnegative) full-sample mean.  It is nonvacuous for squared
loss.  For example, suppose
$\ell_f(i)=(f(\bx_i)-y_i)^2\le L_0$, the set of evaluation vectors generated
by $\cF$ is convex, and $f_n^*$ minimizes the full-sample squared loss.  We set
$d_i=f(\bx_i)-f_n^*(\bx_i)$ and
$e_i=f_n^*(\bx_i)-y_i$.  Write $ed$ and $d^2$ for the functions
$i\mapsto e_id_i$ and $i\mapsto d_i^2$.
Convexity makes the segment of evaluation vectors from $f_n^*$ to $f$
feasible.  For $0\le t\le1$, this segment has squared loss
$n^{-1}\sum_{i=1}^n(e_i+td_i)^2$, whose right derivative at zero is
nonnegative by optimality of $f_n^*$.  Thus
$2\cL_n(ed)\ge0$.  Hence, for $h=\ell_f-\ell_{f_n^*}$,
$\cL_n(h)=\cL_n(d^2)+2\cL_n(ed)\ge\cL_n(d^2)$.
Moreover, $h(i)=d_i\pth{(f(\bx_i)-y_i)+(f_n^*(\bx_i)-y_i)}$ and both residuals
have absolute value at most $\sqrt{L_0}$.  Therefore
$T_n(h)=\cL_n(h^2)\le4L_0\cL_n(d^2)
\le 4L_0\cL_n(h)$,
so Assumption~\ref{assumption:main}(2) holds with $B=4L_0$.
\end{remark}
The bounded-loss condition and Assumption~\ref{assumption:main} are the same
generic assumptions used in the preceding transductive local-complexity
analyses
\citep{yang2025a-concentration-sampling-without-replacement,
yang2026improvedgeneralizationboundstransductive}.  The new concentration and
peeling arguments introduce no additional statistical assumption.  The VC and
kernel applications add only their setting-specific assumptions, which are
listed together with their provenance in
Table~\ref{tab:assumption-provenance}.
For $f\in\cF$, write $g_f\defeq\ell_f-\ell_{f_n^*}$.  For
$h\in\Delta_{\cF}$, define the surrogate localization functional
\bal\label{eq:tTn-def}
\tT_n(h)\defeq
\inf_{\substack{f_1,f_2\in\cF\\ \ell_{f_1}-\ell_{f_2}=h}}
2B\set{\cL_n(g_{f_1})+\cL_n(g_{f_2})}.
\eal
The infimum makes $\tT_n(h)$ depend only on the loss difference $h$,
even when that difference has several representations.
Lemma~\ref{lemma:STLC-delta-ell-f} proves that $\tT_n$ is finite and
nonnegative and satisfies $T_n(h)\le\tT_n(h)$.

We also need to control the full-sample excess losses of $\hat f_m$ and
$\hat f_u$.  For $r>0$, define
$\cF_r^*\defeq\set{h\in\Delta^*_{\cF}\colon B\cL_n(h)\le r}\cup\set{0}$
and $(\cF_r^*)^2\defeq\set{h^2\colon h\in\Delta^*_{\cF},
B\cL_n(h)\le r}\cup\set{0}$.
Thus $r_{u,m}$ below controls arbitrary loss differences, whereas $r^*$
controls excess losses relative to $f_n^*$.
\begin{theorem}[Generic STLC excess-risk bound]\label{theorem:STLC-delta-ell-f-excess-risk-upper-bound}
Suppose that $0\le\ell_f(i)\le L_0$ for every $f\in\cF$ and $i\in[n]$,
that Assumption~\ref{assumption:main} holds, and that the minima defining
$\hat f_m$ and $\hat f_u$ are attained for every split $Z$.  Fix
$K_0>1$.  With $\cH=\Delta_{\cF}$ and $\tT_n$ defined in
(\ref{eq:tTn-def}), let $\psi_{u,m}$ be a sub-root function satisfying
(\ref{eq:STLC-cond-psi-general}), and let $r_{u,m}$ be its positive fixed
point.  Let $\psi^*_{u,m}$ be a sub-root function with positive fixed point
$r^*$ such that, for every $r\ge r^*$,
\bal\label{eq:STLC-cond-um-delta-star-ell-f-psi-star}
\psi^*_{u,m}(r)\ge
\max_{\substack{p\in\set{u,m}\\ \eta\in\set{+,-}}}
\max\set{\cfrakR_p^\eta\pth{\cF_r^*},
\cfrakR_p^\eta\pth{(\cF_r^*)^2}}.
\eal
Let $c_{\Delta}$ be the constant, depending only on $B$ and $L_0$, in
Theorem~\ref{theorem:STLC-delta-star-ell-f}.  Define
$c_3\defeq c_2+4Bc_{\Delta}/K_0$, where $c_1,c_2$ are the constants in
Theorem~\ref{theorem:TLC} with $H_0=L_0$.
Then, for every $x>0$, with probability at least $1-3\exp(-x)$ over $Z$,
the excess risk of $\hat f_m$ satisfies
\bal\label{eq:STLC-ell-f-excess-risk-upper-bound}
\cE(\hat f_m) \le c_1 r_{u,m} +\frac{4Bc_{\Delta}r^*}{K_0}
+\frac{c_3x}{N_{u,m}}.
\eal
\end{theorem}
\noindent\textbf{Proof roadmap for
Theorem~\ref{theorem:STLC-delta-ell-f-excess-risk-upper-bound}.}
In Appendix~\ref{sec:proof-generic-STLC},
Lemma~\ref{lemma:STLC-delta-ell-f} first validates the surrogate on
$\Delta_{\cF}$.
Lemma~\ref{lemma:concentration-gu-gm-Delta*} applies Theorem~\ref{theorem:TLC}
to $\Delta^*_{\cF}$ in both split directions.
Empirical optimality and absorption give
Theorem~\ref{theorem:STLC-delta-star-ell-f}, controlling both full-sample
excess losses by $r^*+x/N_{u,m}$.
Intersecting these two events with the STLC event on $\Delta_{\cF}$
allows substitution of $\ell_{\hat f_m}-\ell_{\hat f_u}$.
Its training average is nonpositive by optimality, and its surrogate is
at most $2B$ times the sum of the controlled excess losses.
This proves (\ref{eq:STLC-ell-f-excess-risk-upper-bound}) with failure
probability at most $3\exp(-x)$.
The bound in (\ref{eq:STLC-ell-f-excess-risk-upper-bound}) has the same
fixed-point and confidence structure as the inductive bound
(\ref{eq:conceptual-risk-bound-inductive}), up to constants multiplying
$r_{u,m}$ and $r^*$.  In contrast,
\citep[Theorem 3.6]{yang2025a-concentration-sampling-without-replacement}
requires $u\gg m^2$ or $m\gg u^2$, while
\citep[Theorem 4.11]{yang2026improvedgeneralizationboundstransductive}
has the additional factor $\log_2(4N_{u,m}/\delta)$ in its confidence term.

\section{Applications}
\label{sec:applications}
We apply the generic STLC bounds to realizable binary classification in
Section~\ref{sec:transductive-learning-finite-VC-dim} and to Transductive Kernel
Learning (TKL) in Section~\ref{sec:TKL}.  In the realizable setting, one
predictor agrees with every full-sample label.  Detailed proofs are deferred to
Appendix~\ref{sec:proofs}.
\subsection{Transductive Learning Over Binary-Valued Function Classes}
\label{sec:transductive-learning-finite-VC-dim}
We derive an excess-risk bound for transductive learning over a binary-valued
class $\cF$ of finite VC dimension using Theorem~\ref{theorem:TLC}.
\begin{theorem}
[Upper bound for realizable transductive learning with finite VC dimension]
\label{theorem:STLC-delta-ell-f-excess-risk-upper-bound-VC-dim}
Let $\cF$ be a class of $\set{0,1}$-valued functions on $\cX$ with
VC-dimension $2\le\dVC<\infty$, and suppose that $u\ge m\ge\dVC$.
For $i\in[n]$, let $y_i\in\set{0,1}$ and
$\ell_f(i)=\pth{f(\bx_i)-y_i}^2=\indict{f(\bx_i)\ne y_i}$.
Assume that the full sample is realizable: there exists $f^*\in\cF$ such
that $f^*(\bx_i)=y_i$ for every $i\in[n]$.  For each split, let
$\hat f_m$ be any empirical minimizer of $\cL_m(\ell_f)$ over $f\in\cF$.
There are absolute positive constants $c''_0$ and $c'_1$ such that, for every
$x>0$, with probability at least $1-\exp(-x)$ over $Z$,
\bal\label{eq:STLC-delta-ell-f-excess-risk-upper-bound-VC-dim}
\cL_u(\ell_{\hat f_m})\le c''_0\frac{\dVC\log(me/\dVC)}{m}
+ c'_1\frac{x}{m}.
\eal
\end{theorem}
\noindent\textbf{Proof roadmap for
Theorem~\ref{theorem:STLC-delta-ell-f-excess-risk-upper-bound-VC-dim}.}
Realizability turns every excess loss into an indicator and gives
$T_n(h)=\cL_n(h)$.  A VC relative-deviation inequality transfers the
population localization $\cL_n(h)\le r$ to an empirical $L_2$ localization
for a with-replacement sample.  Symmetrization, contraction, and Dudley's
entropy integral then bound each localized Rademacher process by a constant
multiple of $\sqrt{a_m r}+a_m$, where
$a_m=\dVC\log(me/\dVC)/m$.  This is a sub-root majorant with fixed point of
order $a_m$.  Theorem~\ref{theorem:TLC}, used with $K_0=2$, then absorbs half
of the full-sample error into the test error.  The zero training error of
$\hat f_m$ yields the displayed bound.

In the realizable case, $\cL_m(\ell_{\hat f_m})=0$ and
$\cE(\hat f_m)=\cL_u(\ell_{\hat f_m})$.  Taking
$x=\dVC\log(me/\dVC)$ in the theorem gives
$\cL_u(\ell_{\hat f_m})\le
\cO\pth{\dVC\log(me/\dVC)/m}$.
For $m\ge9$, this is within a logarithmic factor of the expected minimax
lower bound $(\dVC-1)/(16m)$ in
\citep[Theorem 3]{tolstikhin2016minimaxlowerboundsrealizable}.  Four rates
should be distinguished.  The classical LRC-based, standard ERM-style
inductive upper bound in \citep[Corollary~3.7]{bartlett2005} has order
$\dVC\log(me/\dVC)/m$.  The optimal realizable inductive PAC rate has order
$\dVC/m$.  The transductive upper bound in
(\ref{eq:STLC-delta-ell-f-excess-risk-upper-bound-VC-dim}) has the former
order, while the transductive minimax lower bound has the latter order.  Thus
the theorem matches the classical LRC-based inductive upper bound, not the
optimal inductive PAC rate. We further note that \citet{yang2026improvedgeneralizationboundstransductive} gives a
log-free rate of $\cO(\dVC/m)$ using randomized learners in a transductive setup.

\subsection{Transductive Kernel Learning}
\label{sec:TKL}
\subsubsection{Background in RKHS and Kernel Learning}

Let $K\colon\cX\times\cX\to\RR$ be a positive definite kernel with
RKHS $\cH_K$.  We use ``positive definite'' in the non-strict sense, so every
finite Gram matrix is positive semidefinite.  For the fixed full sample, define
the Gram matrix
$\bK\in\RR^{n\times n}$ by $\bK_{ij}=K(\bx_i,\bx_j)$ and the
finite-dimensional subspace
$\cH_{\bX_n}\defeq\Span\set{K(\cdot,\bx_i)\colon i\in[n]}\subseteq\cH_K$.
For $\mu\ge0$, define the transductive RKHS ball
\bals
\cH_{\bX_n}(\mu) &\defeq \set{f \in \cH_{\bX_n} \colon \norm{f}{\cH_K}\le\mu} =\set{\sum\limits_{i=1}^n K(\cdot,\bx_i)\alpha_i\colon
\balpha^{\top}\bK\balpha\le\mu^2}.
\eals
The identity follows from the reproducing property.  If $\bK$ is singular,
different coefficient vectors may represent the same RKHS element.  In TKL,
the prediction class is $\cF=\cH_{\bX_n}(\mu)$.

\subsubsection{Results}
\label{sec:TKL-results}
The TKL bound uses the minimizer-existence condition in
Assumption~\ref{assumption:main}(1), a bounded loss, and the following
Lipschitz and prediction-localization conditions.
\begin{assumption}
\label{assumption:Lipschitz-loss-Tn-f-Ln-ellf}
\begin{itemize}[leftmargin=18pt]
\item[(1)] There is a constant $L\ge0$ such that, for every
$a,b\in\RR$ and $y\in\cY$, the loss satisfies
$\abth{\ell(a,y)-\ell(b,y)}\le L\abth{a-b}$.
\item[(2)] There is a constant $B'>0$ such that, for every $f\in\cF$,
$T_n\pth{f-f_n^*}\le
B'\cL_n\pth{\ell_f-\ell_{f_n^*}}$.
\end{itemize}
\end{assumption}
Under Assumption~\ref{assumption:main}(1),
Assumption~\ref{assumption:Lipschitz-loss-Tn-f-Ln-ellf} implies
Assumption~\ref{assumption:main}(2).  Indeed, for
$g_f=\ell_f-\ell_{f_n^*}$,
$T_n(g_f)\le L^2T_n(f-f_n^*)\le L^2B'\cL_n(g_f)$.
Thus Assumption~\ref{assumption:main}(2) holds with
$B=B_K\defeq\max\set{1,L^2B'}$.  We use this value below.
The next result is obtained by bounding the two fixed points in the generic
STLC excess-risk bound, Theorem~\ref{theorem:STLC-delta-ell-f-excess-risk-upper-bound},
in terms of the empirical kernel spectrum.
\begin{theorem}[STLC bound for transductive kernel learning]
\label{theorem:STLC-kernel}
Let $\cF=\cH_{\bX_n}(\mu)$, and suppose that
Assumption~\ref{assumption:main}(1) and
Assumption~\ref{assumption:Lipschitz-loss-Tn-f-Ln-ellf} hold.  Assume that $K$ is a
positive definite kernel and that, for every $f\in\cF$ and $i\in[n]$,
$0\le\ell_f(i)\le L_0$, where $L_0>0$.  Fix $K_0>1$.  Let
$\hat\lambda_1\ge\cdots\ge\hat\lambda_n\ge0$ be the eigenvalues of
$\bK/n$.  For $Q\in\set{0,\ldots,n}$, define
\bals
r(u,m,Q) \defeq Q \pth{\frac{1}{u} + \frac{1}{m}} + \sqrt{\frac{\sum\limits_{q = Q+1}^{n}\hat \lambda_q}{u}}
+ \sqrt{\frac{\sum\limits_{q = Q+1}^{n}\hat \lambda_q}{m}}.
\eals
There is a positive constant $c_5$, depending only on
$K_0,B',L_0,L$, and $\mu$, such that, for every $x>0$, with probability
at least $1-3\exp(-x)$ over $Z$,
\bal\label{eq:STLC-kernel-excess-loss}
\cE(\hat f_m) \le
c_5 \pth{\min_{Q\in\set{0,\ldots,n}} r(u,m,Q)
+ \frac{x}{N_{u,m}}}.
\eal
For every split, the minimizers $\hat f_m$ and $\hat f_u$ exist because
$\cH_{\bX_n}(\mu)$ is compact and the two empirical risks are continuous.
\end{theorem}
\noindent\textbf{Proof roadmap for Theorem~\ref{theorem:STLC-kernel}.}
Lemma~\ref{lemma:STLC-kernel-ind} expands functions in an eigenbasis of the
empirical covariance operator.  Splitting that expansion after coordinate
$Q$ bounds each localized Rademacher process by a head term of order
$\sqrt{Qr/p}$ and a spectral-tail term of order
$\sqrt{p^{-1}\sum_{q>Q}\hat\lambda_q}$ for $p\in\set{u,m}$.
Lipschitz contraction transfers this prediction-process bound to loss
differences.  The Lipschitz and prediction-localization conditions first imply
the loss Bernstein condition and then convert loss localization into
prediction localization.  The resulting sub-root majorants control both
fixed points in Theorem~\ref{theorem:STLC-delta-ell-f-excess-risk-upper-bound}.
Solving their fixed-point inequalities and minimizing over $Q$ gives
(\ref{eq:STLC-kernel-excess-loss}).
\begin{example}[Worked kernel fixed point]
If $L=0$ or $\cH_{\bX_n}(\mu)=\set{0}$, then $\cE(\hat f_m)=0$, as in the
appendix proof.  Otherwise, the majorants below are nonzero and have positive
fixed points.  We use exactly the quantities constructed in that proof.  Put
$C_{L_0}\defeq2\max\set{1,2L_0}$, $B_1\defeq B'/B_K$,
$\beta\defeq B_1/4$, and $A\defeq4C_{L_0}L$.  For
$Q\in\set{0,\ldots,n}$, let
$S_Q\defeq\sum_{q=Q+1}^n\hat\lambda_q$,
$D_Q\defeq\sqrt{\beta Q/u}+\sqrt{\beta Q/m}$, and
$E_Q\defeq\mu\pth{\sqrt{S_Q/u}+\sqrt{S_Q/m}}$.  The contraction and
spectral split in Lemma~\ref{lemma:STLC-kernel-ind} give the sub-root majorant
$\varphi_{u,m}$ constructed in the proof, and for every
$Q\in\set{0,\ldots,n}$,
\bals
\varphi_{u,m}(r)\le A\pth{D_Q\sqrt r+E_Q}.
\eals
At its positive fixed point, the Cauchy--Schwarz inequality in the form
$2st\le s^2+t^2$, applied with $s=\sqrt{r_{u,m}}$ and $t=AD_Q$, and the bound
$D_Q^2\le2\beta Q(1/u+1/m)$ give
\bals
r_{u,m}&\le AD_Q\sqrt{r_{u,m}}+AE_Q
\le\frac12r_{u,m}+\frac12A^2D_Q^2+AE_Q,\\
r_{u,m}&\le A^2D_Q^2+2AE_Q\le c_Kr(u,m,Q),
\eals
where $c_K$ depends only on $B',L_0,L$, and $\mu$.  The last inequality holds
for every $Q$, hence
$r_{u,m}\le c_K\min_Qr(u,m,Q)$, exactly as in
(\ref{eq:STLC-kernel-fixed-point-bound}).  The excess-loss majorant constructed
in the proof has fixed point $r^*\le2r_{u,m}$.  Substitution into
Theorem~\ref{theorem:STLC-delta-ell-f-excess-risk-upper-bound} yields
(\ref{eq:STLC-kernel-excess-loss}).  Thus the fixed points certify the bound
but need not be computed by the learning algorithm.
\end{example}
A detailed comparison of (\ref{eq:STLC-kernel-excess-loss}) with existing results, including \citep[Corollary 14]{TolstikhinBK14-local-complexity-TRC},
\citep[Theorem 6.2]{yang2026improvedgeneralizationboundstransductive}, and \citep[Theorem 4.1]{yang2025a-concentration-sampling-without-replacement}, is provided in
Section~\ref{sec:more-comparison-TKL} of the appendix.

\section{Conclusion}
We introduced Sharper Transductive Local Complexity, based on a Bernstein-type
concentration inequality for the test--train supremum under uniform sampling
without replacement.  Peeling with a surrogate localization functional gives
generic bounded-loss excess-risk bounds with a sharp confidence term.  The
method also yields a realizable VC bound that is minimax-optimal up to the
factor $\log(me/\dVC)$ and a spectrum-adaptive kernel-learning bound without
the earlier imbalance prefactors.

\begin{ack}
This work was supported by National Institute of Health 1OT2OD037955-01.
\end{ack}

\bibliographystyle{unsrtnat}
\bibliography{ref}


\appendix

\section{Mathematical Tools}
\label{sec::math-tools}

We collect the empirical-process facts used in the VC and kernel arguments.
For $p\in\set{u,m}$, let
$\bY^{(p)}=(Y_1,\ldots,Y_p)$ have independent coordinates uniformly
distributed on $[n]$, and let
$\bsigma=(\sigma_1,\ldots,\sigma_{\max\set{u,m}})$ be independent of these
vectors.  For a function $h$ on $[n]$, define
$R^{(\textup{ind})}_{\bsigma,\bY^{(p)}}h
\defeq p^{-1}\sum_{i=1}^p\sigma_i h(Y_i)$.  For a nonempty class $\cA$ of
such functions, its with-replacement Rademacher complexity is
$\cfrakR^{(\textup{ind})}_p(\cA)
\defeq\Expect{\bY^{(p)},\bsigma}{\sup_{h\in\cA}
R^{(\textup{ind})}_{\bsigma,\bY^{(p)}}h}$.

\begin{theorem}[Contraction property of inductive Rademacher complexity \citep{LedouxTal91-probability}]
\label{theorem:contraction-RC}
Let $\cH$ be a nonempty class of real-valued functions on $[n]$.
Suppose that, for some $L\ge0$ and every $i\in[n]$, the map
$\phi_i\colon\RR\to\RR$ satisfies $\phi_i(0)=0$ and
$\abth{\phi_i(a)-\phi_i(b)}\le L\abth{a-b}$ for all $a,b\in\RR$.
Define
$\phi\circ\cH\defeq
\set{i\mapsto\phi_i\pth{h(i)}\colon h\in\cH}$.  Then
$\cfrakR^{(\textup{ind})}_p(\phi\circ\cH)
\le L\cfrakR^{(\textup{ind})}_p(\cH)$ for every $p\in\set{u,m}$.
The usual scalar contraction inequality is recovered by taking $\phi_i=\phi$
for every $i\in[n]$.
\end{theorem}

For a subset $A$ of a pseudometric space $(S,d)$ and $\eps>0$, let
$N(A,d,\eps)$ denote the external covering number, namely, the minimum
cardinality of a set $E\subseteq S$ such that
$\inf_{e\in E}d(a,e)\le\eps$ for every $a\in A$.  Its value is $\infty$ if
no such finite set exists.  In particular, $A\subseteq B\subseteq S$ implies
$N(A,d,\eps)\le N(B,d,\eps)$.
The covering numbers below are taken in the ambient vector space of
real-valued functions on the fixed design.

\begin{theorem}
[Dudley's integral entropy bound~\citep{dudley1999uniform}]
\label{theorem:dudley-integral-entropy-bound}
Let $\cF$ be a nonempty function class defined on $\cX$, and let
$\bx_1,\ldots,\bx_n\in\cX$ be a fixed design.  Define
$P_n\defeq n^{-1}\sum_{i=1}^n\delta_{\bx_i}$, where $\delta_{\bx}$ is the
unit point mass at $\bx$.  Put
$\norm{f}{L_2(P_n)}^2\defeq n^{-1}\sum_{i=1}^nf^2(\bx_i)$ and
$T\defeq\sup_{f,g\in\cF}\norm{f-g}{L_2(P_n)}$.  Let
$\bsigma\defeq(\sigma_1,\ldots,\sigma_n)$ have independent Rademacher
coordinates.  Then there exists an absolute constant $C_0$ such that
\bals
\Expect{\bsigma}{\sup_{f\in\cF}\frac1n\sum_{i=1}^n\sigma_i f(\bx_i)}
\le\frac{C_0}{\sqrt n}\int_0^{T/2}
\sqrt{\log N\pth{\cF,\norm{\cdot}{L_2(P_n)},\delta}}\diff\delta.
\eals
\end{theorem}
\begin{proposition}
[Concentration of a function class with finite VC-dimension, adapted from
{\citep[Corollary 3.5 and the proof of Corollary 3.7]{bartlett2005}}]
\label{proposition:concentration-vc-class-ind}
Let $P$ be a probability measure on a measurable space $\cX$, and let
$\cF$ be a nonempty class of measurable $\set{0,1}$-valued functions on
$\cX$ with VC-dimension at most $\dVC$, where $1\le\dVC\le n$.
Assume that $\cF$ is pointwise separable, meaning that it has a countable
subclass $\cF_0$ such that every $f\in\cF$ is a pointwise limit of a sequence
from $\cF_0$.  Let $\bx_1,\ldots,\bx_n$ be independent with common law $P$,
and let $P_n\defeq n^{-1}\sum_{i=1}^n\delta_{\bx_i}$ be their empirical
measure, where $\delta_{\bx}$ denotes the unit point mass at $\bx$.
Then there is an absolute positive constant $c$ such that
for all $K_0>1$ and every $x>0$, with probability at least $1-\exp(-x)$,
\bal\label{eq:concentration-vc-class-ind}
\Expect{\bx\sim P_n}{f(\bx)}-\frac{K_0+1}{K_0}\Expect{\bx\sim P}{f(\bx)}
\le c K_0 \pth{\frac{\dVC \log (ne/\dVC)}{n}+\frac xn}, \quad \forall f\in\cF.
\eal
\end{proposition}
\begin{proof}
\noindent\textit{Proof outline.}  Apply the empirical-to-population direction
of the cited local-Rademacher inequality to the countable approximating class,
bound its fixed point by the Sauer entropy estimate, and extend the result
to $\cF$ by pointwise approximation and bounded convergence.

Apply Corollary~3.5 of \citep{bartlett2005} to $\cF_0$ and use the second,
upper-deviation inequality displayed there.  The calculation in the proof
of Corollary~3.7 of that work bounds its fixed point by
$C\dVC\log(ne/\dVC)/n$ for an absolute constant $C$.  Substitution and a
change of the absolute constant give (\ref{eq:concentration-vc-class-ind})
simultaneously on $\cF_0$ on a measurable event of probability at least
$1-\exp(-x)$, since $\cF_0$ is countable.  For $f\in\cF$, choose a sequence
$(f_j)$ in $\cF_0$ converging pointwise to $f$.  Bounded convergence gives
$\Expect{\bx\sim P}{f_j(\bx)}\to\Expect{\bx\sim P}{f(\bx)}$, and convergence
of the finite empirical sums gives the corresponding convergence under
$P_n$ for every sample.  Passing to the limit on the same event proves the
inequality simultaneously for all $f\in\cF$.  On a finite domain such as
$[n]$, pointwise separability is automatic because there are only finitely
many distinct binary-valued functions.
\end{proof}

The following theorem relates transductive complexity to its with-replacement
counterpart.
\begin{theorem}
[{\citep[Theorem 2.1]{yang2025a-concentration-sampling-without-replacement}, \citep[Theorem 3.1]{yang2026improvedgeneralizationboundstransductive}}]
\label{theorem:TC-RC}
Let $\cH$ be a nonempty class of real-valued functions on $[n]$.  Then
$\max\set{\cfrakR^+_u(\cH),\cfrakR^-_u(\cH)}
\le2\cfrakR^{(\textup{ind})}_u(\cH)$ and
$\max\set{\cfrakR^+_m(\cH),\cfrakR^-_m(\cH)}
\le2\cfrakR^{(\textup{ind})}_m(\cH)$.
\end{theorem}

\section{Concentration of Centered Suprema Under Sampling Without Replacement}
\label{sec:concentration-general-sup-empirical}

We first record the consequence of Theorem~\ref{theorem:concentration-gd-STLC}
for the two one-sided empirical processes centered at the full-sample mean.
The statement includes both an individual tail bound and a simultaneous bound
at the same prescribed failure probability.

\begin{corollary}[Centered one-sided suprema]
\label{corollary:concentration-gu}
Fix integers $u,m\ge2$, put $n=u+m$, and let $Z$ be uniformly distributed
over the $u$-element subsets of $[n]$.  Write $\bar Z=[n]\setminus Z$.  Let
$\cH$ be a nonempty class of real-valued functions on $[n]$ and suppose that,
for some $H_0,r>0$,
$\sup_{h\in\cH}\max_{i\in[n]}\abth{h(i)}\le H_0$ and
$\sup_{h\in\cH}T_n(h)\le r$.  Define
$g^+_u(\cH,Z)\defeq\sup_{h\in\cH}\set{\cL_h(Z)-\cL_n(h)}$ and
$g^-_u(\cH,Z)\defeq\sup_{h\in\cH}\set{\cL_n(h)-\cL_h(Z)}$.
Use the convention $\cfrakR^+_{N_{u,m}}=\cfrakR^+_u$ when $u\le m$ and
$\cfrakR^+_{N_{u,m}}=\cfrakR^+_m$ when $m<u$.  When $u=m$, the two
quantities agree by complementation.  With $c_0=\sqrt{219}$ as in
Theorem~\ref{theorem:concentration-gd-STLC}, put
$\Gamma_{u,m}(s,\cH)\defeq c_0\sqrt{rs/N_{u,m}}
+(c_0/2)\cfrakR^+_{N_{u,m}}(\cH^2)
+(16H_0+c_0/2)s/N_{u,m}$ for $s>0$.
Then, for each $\eta\in\set{+,-}$ and
every $x>0$,
\bal\label{eq:concentration-gu-individual}
\Prob{g_u^\eta(\cH,Z)-\Expect{}{g_u^\eta(\cH,Z)}
>\frac{m}{n}\Gamma_{u,m}(x,\cH)}
\le \exp(-x).
\eal
Moreover, for every $x>0$, with probability at least $1-\exp(-x)$,
\bal\label{eq:concentration-gu}
\max_{\eta\in\set{+,-}}
\set{g_u^\eta(\cH,Z)-\Expect{}{g_u^\eta(\cH,Z)}}
\le \frac{m}{n}\Gamma_{u,m}(x+\log 2,\cH).
\eal
\end{corollary}

\begin{proof}
\noindent\textit{Proof outline.}  Express each centered one-sided process as
a deterministic multiple of a test--train process, apply
Theorem~\ref{theorem:concentration-gd-STLC} to $\cH$ and $-\cH$, and combine
the two tails by a union bound.
For every $h\in\cH$, the identity
$\cL_n(h)=u\cL_h(Z)/n+m\cL_h(\bar Z)/n$ gives
$\cL_h(Z)-\cL_n(h)=\frac{m}{n}\set{\cL_h(Z)-\cL_h(\bar Z)}$ and
$\cL_n(h)-\cL_h(Z)=\frac{m}{n}\set{\cL_h(\bar Z)-\cL_h(Z)}$.
Taking suprema yields $g_u^+(\cH,Z)=\frac{m}{n}g(\cH,Z)$ and
$g_u^-(\cH,Z)=\frac{m}{n}g(-\cH,Z)$.
The classes $\cH$ and $-\cH$ have the same envelope and the same
second-moment radius, and $(-\cH)^2=\cH^2$.  Applying
Theorem~\ref{theorem:concentration-gd-STLC} to $\cH$ and to $-\cH$,
respectively, and then using these identities, including after taking
expectations, proves
(\ref{eq:concentration-gu-individual}).  Finally, apply the two individual
bounds with $x$ replaced by $x+\log 2$.  The union bound gives total failure
probability at most $2\exp\set{-(x+\log 2)}=\exp(-x)$, proving
(\ref{eq:concentration-gu}).
\end{proof}

\subsection{A centered formulation}
\label{sec:centered-concentration-formulation}

The empirical processes in Corollary~\ref{corollary:concentration-gu} are
invariant under centering each function by its full-sample mean.  More
precisely, define $h^\circ(i)\defeq h(i)-\cL_n(h)$ for $i\in[n]$, put
$\cH^\circ\defeq\set{h^\circ\colon h\in\cH}$, and set
$H_\circ\defeq\sup_{h\in\cH}\max_{i\in[n]}\abth{h^\circ(i)}$ and
$\sigma^2\defeq\sup_{h\in\cH}T_n(h^\circ)$.  Then
$\cL_n(h^\circ)=0$ and
$g_u^\eta(\cH,Z)=g_u^\eta(\cH^\circ,Z)$ for $\eta\in\set{+,-}$.
If $\sigma=0$, then $T_n(h^\circ)=0$ for every $h\in\cH$, so every
$h^\circ$ vanishes on $[n]$ and the bound below is immediate.  If
$\sigma>0$ and $H_\circ<\infty$, Corollary~\ref{corollary:concentration-gu}
applies to $\cH^\circ$ with $r=\sigma^2$.  Consequently, whenever
$H_\circ<\infty$, with probability at least $1-\exp(-x)$,
\bal\label{eq:concentration-gu-centered}
\max_{\eta\in\set{+,-}}
\set{g_u^\eta(\cH,Z)-\Expect{}{g_u^\eta(\cH,Z)}}
&\le \frac{m}{n}\left[
c_0\sigma\sqrt{\frac{x+\log2}{N_{u,m}}}
+\frac{c_0}{2}\cfrakR^+_{N_{u,m}}\pth{(\cH^\circ)^2}
\right.\nonumber\\
&\hspace{34mm}\left.
+\frac{(16H_\circ+c_0/2)(x+\log2)}{N_{u,m}}
\right].
\eal
This is the appropriate form for comparison with concentration inequalities
whose hypotheses are stated for centered function classes.

\subsection{The expectation gap between the two sampling schemes}
\label{sec:expectation-gap-comparison}
We justify the estimate used in Section~\ref{sec:comparison-earlier-bounds}.
Retain the centered class $\cH^\circ$ above and suppose $H_\circ\le1$.
Let $\bY^{(u)}=(Y_1,\ldots,Y_u)$ consist of independent uniform indices in
$[n]$.  The process $\widetilde g_u(\bY^{(u)},\cH)$ averages over these
indices with replacement, whereas $g_u^+(\cH,Z)$ averages over a uniform
subset of size $u$.

\begin{proof}[\textup{Proof of (\ref{eq:with-without-expectation-gap})}]
\noindent\textit{Proof outline.}  Couple the independent indices with an
ordered sample without replacement, count the coordinates altered by this
coupling, and bound the change of the supremum per altered coordinate.

\noindent\textit{Step 1. Construct the sample without replacement.}
Recursively construct distinct indices $J_1,\ldots,J_u$.
At step $k$, put $A_{k-1}=\set{J_1,\ldots,J_{k-1}}$, with $A_0=\emptyset$.
If $Y_k\notin A_{k-1}$, set $J_k=Y_k$.  Otherwise choose $J_k$ uniformly
from $[n]\setminus A_{k-1}$, using fresh randomness independent of all the
$Y_i$ and of the previous auxiliary choices.
Given the preceding history, $Y_k$ is still uniform on $[n]$.  Hence, for
each $j\notin A_{k-1}$, the conditional probability that $J_k=j$ equals
\bals
\frac1n+\frac{k-1}{n}\frac1{n-k+1}=\frac1{n-k+1}.
\eals
Thus $(J_1,\ldots,J_u)$ is an ordered uniform sample without replacement,
and $\set{J_1,\ldots,J_u}$ has the law of $Z$.
Moreover, $J_k\ne Y_k$ holds exactly when $Y_k\in A_{k-1}$, so
$\Prob{J_k\ne Y_k}=(k-1)/n$.

\noindent\textit{Step 2. Bound the change of the supremum.}
For $\bj=(j_1,\ldots,j_u)\in[n]^u$, set
$F(\bj)=\sup_{h\in\cH}u^{-1}\sum_{k=1}^u h^\circ(j_k)$ and write
$\bJ=(J_1,\ldots,J_u)$.
Since $\abth{h^\circ(i)}\le1$, replacing one coordinate of $\bj$ changes
$F(\bj)$ by at most $2/u$.  Consequently,
\bals
\abth{\Expect{}{F(\bY^{(u)})}-\Expect{}{F(\bJ)}}
&\stackrel{\circled{1}}{\le}
\frac2u\sum_{k=1}^u\Prob{J_k\ne Y_k}
=\frac2u\sum_{k=1}^u\frac{k-1}{n}
=\frac{u-1}{n}.
\eals
Step \circled{1} follows by changing the differing coordinates one at a
time, taking absolute values, and then expectations.
The two expectations are those in (\ref{eq:with-without-expectation-gap}).
Their difference is nonnegative by Hoeffding's comparison principle, as
used in \citep[Lemma~3]{TolstikhinBK14-local-complexity-TRC}.
This also holds for an arbitrary bounded class on $[n]$: for each of the
finitely many sequences in $[n]^u$, choose an $\eps$-maximizer in $\cH$.
The finite class of all selected functions approximates every supremum
above within $\eps$.  Apply the finite-class comparison and let
$\eps\downarrow0$.  This proves the claimed two-sided estimate.
\end{proof}

\section{Detailed Proofs}
\label{sec:proofs}
The proofs follow the same order as the main results.
Appendix~\ref{sec:modified-log-sobolev-prelim} gives the entropy inequality
for the swap walk.  Appendix~\ref{sec:proof-concentration-STLC} uses it to
prove Theorem~\ref{theorem:concentration-gd-STLC}.
Appendix~\ref{sec:proofs-theorem-STLC-STLC-nonnegative-func-class} then proves
Theorem~\ref{theorem:TLC}, and Appendix~\ref{sec:proof-generic-STLC} derives
Theorem~\ref{theorem:STLC-delta-ell-f-excess-risk-upper-bound}.
The remaining subsections establish the VC and kernel applications.

\subsection{Background in Modified Log-Sobolev Inequality}
\label{sec:modified-log-sobolev-prelim}

Fix $1\le u\le n-1$, put $m=n-u$, and let
$\Omega\defeq\set{Z\subseteq[n]\colon \abth{Z}=u}$ and
$\pi(Z)\defeq\binom{n}{u}^{-1}$.
The non-lazy swap walk on $\Omega$ has transition matrix
\bals
P(Z,Z')=
\begin{cases}
(um)^{-1},& \abth{Z\cap Z'}=u-1,\\
0,&\text{otherwise}.
\end{cases}
\eals
Equivalently, from $Z$ it chooses $(i,j)$ uniformly from
$Z\times([n]\setminus Z)$ and moves to
$Z^{(ij)}=(Z\setminus\set{i})\cup\set{j}$.  The chain is irreducible, and
$P(Z,Z')=P(Z',Z)$. Hence it is reversible with stationary distribution
$\pi$.

For functions $f,g\colon\Omega\to\RR$, define
\bals
\cE_P(f,g)\defeq\frac12\sum_{Z,Z'\in\Omega}\pi(Z)P(Z,Z')
\bigl(f(Z)-f(Z')\bigr)\bigl(g(Z)-g(Z')\bigr).
\eals
For $f\colon\Omega\to\RR_{\ge0}$, let
$\Ent_\pi(f)\defeq\Expect{\pi}{f\log f}
-\Expect{\pi}{f}\log\Expect{\pi}{f}$, where $0\log0=0$.  The modified
log-Sobolev constant is
\bals
\rho(P)\defeq
\inf_{\substack{f\colon\Omega\to\RR_{>0}\\\Ent_\pi(f)\ne0}}
\frac{\cE_P(f,\log f)}{\Ent_\pi(f)}.
\eals
\begin{theorem}
[{Modified log-Sobolev inequality, see
\citep[Theorem 1]{Cryan2021-modified-log-sobolev-inequality}}]
\label{theorem:modified-log-sobolev}
For the non-lazy swap walk defined above, $\rho(P)\ge1/u$.  Consequently,
for every $f\colon\Omega\to\RR_{>0}$,
\bal\label{eq:mlsi-swap}
\Ent_\pi(f)\le u\,\cE_P(f,\log f).
\eal
\end{theorem}
\begin{proof}
\noindent\textit{Proof outline.}  Identify the uniform law on $u$-subsets
with the bases measure of the uniform matroid, compare the cited lazy down--up
walk with the non-lazy swap walk, and rescale the Dirichlet form.
The measure $\pi$ is the uniform measure on the bases of the uniform matroid
of rank $u$.  Its generating polynomial is a positive multiple of the
elementary symmetric polynomial $e_u$ and is $u$-homogeneous and strongly
log-concave.  Thus Theorem~1 of the cited work applies to its bases-exchange
(equivalently, down-up) walk, which we denote by $P_{\mathrm{DU}}$.  From a set
$Z\in\Omega$, this walk first deletes a uniformly chosen element of $Z$ and
then adds a uniformly chosen element among the $m+1$ elements outside the
resulting $(u-1)$-set.  Consequently,
\bals
P_{\mathrm{DU}}(Z,Z')=
\begin{cases}
1/(u(m+1)),&\abth{Z\cap Z'}=u-1,\\
1/(m+1),&Z'=Z,\\
0,&\text{otherwise},
\end{cases}
\eals
and hence $P_{\mathrm{DU}}=mP/(m+1)+\bI_{\Omega}/(m+1)$, where
$\bI_{\Omega}$ is the identity transition matrix on $\Omega$.
The cited result gives $\rho(P_{\mathrm{DU}})\ge1/u$.  Self-loops make no
contribution to the Dirichlet form, so
$\cE_{P_{\mathrm{DU}}}(f,\log f)=m\cE_P(f,\log f)/(m+1)$.
Taking the infimum over positive nonconstant $f$ yields
$\rho(P)=\frac{m+1}{m}\rho(P_{\mathrm{DU}})
\ge(m+1)/(mu)\ge1/u$.
The final assertion is the definition of $\rho(P)$ rewritten using this
lower bound.
\end{proof}

\subsection{Proof of Theorem~\ref{theorem:concentration-gd-STLC}
and Corollary~\ref{corollary:concentration-gu}}
\label{sec:proof-concentration-STLC}

Recall that $\Omega$ is the set of $u$-subsets of $[n]$, $\pi$ is its
uniform law, and $P$ is the swap transition matrix from
Appendix~\ref{sec:modified-log-sobolev-prelim}.
For $a\in\RR$, set $a_+\defeq\max\set{a,0}$.  The oriented one-step
variance of $q\colon\Omega\to\RR$ at $Z$ is
$\mathcal V_+(q,Z)\defeq\sum_{Z'\in\Omega}P(Z,Z')
[q(Z)-q(Z')]_+^2$.
Only decreases from $Z$ contribute to this quantity.  The next lemma
connects it to the entropy inequality.

\begin{lemma}[Exponential Dirichlet-form bound]
\label{lemma:dirichlet-upper-bound}
For every $q\colon\Omega\to\RR$ and every $\gamma\ge0$,
\bal\label{eq:dirichlet-upper-bound-cR}
\cE_P\pth{\exp\pth{\gamma q},\gamma q}
\le\gamma^2\Expect{\pi}{\exp\pth{\gamma q(Z)}\mathcal V_+(q,Z)}.
\eal
\end{lemma}
\begin{proof}
\noindent\textit{Proof outline.}  Orient every neighboring pair toward its
larger $q$-value, use $1-\exp(-d)\le d$, and sum the resulting squared positive
increments.
By reversibility, the factor $1/2$ in the definition of the Dirichlet form
allows us to sum once over unordered neighboring pairs.  Fix such a pair
and orient it so that $q(Z)\ge q(Z')$.  With
$d=\gamma(q(Z)-q(Z'))\ge0$, its contribution to the Dirichlet form is
$\pi(Z)P(Z,Z')\bth{\exp\pth{\gamma q(Z)}-\exp\pth{\gamma q(Z')}}d
=\pi(Z)P(Z,Z')\exp\pth{\gamma q(Z)}\pth{1-\exp(-d)}d$.
The elementary inequality $1-\exp(-d)\le d$ bounds this by
$\gamma^2\pi(Z)P(Z,Z')\exp\pth{\gamma q(Z)}
\bigl(q(Z)-q(Z')\bigr)^2$.
Summing over the oriented unordered pairs gives precisely the right-hand
side of (\ref{eq:dirichlet-upper-bound-cR}).
\end{proof}

The next estimate is the main link between the sampling geometry and
localization.  It bounds the random variance of the original supremum by a
supremum over squared functions, then controls that auxiliary process in
the same way.
\begin{lemma}[Oriented swap-variance bounds]
\label{lemma:oriented-swap-variance-bounds}
Assume $u\le m$ and let $\cH$ satisfy the hypotheses of
Theorem~\ref{theorem:concentration-gd-STLC}.  Set
$G(Z)\defeq g(\cH,Z)$ and
$Q(Z)\defeq g_u^+(\cH^2,Z)
=\sup_{h\in\cH}\pth{\cL_{h^2}(Z)-T_n(h)}$.
Then, for every $Z\in\Omega$,
\bal
\mathcal V_+(G,Z)
&\le \frac{24r}{u^2}+\frac{8Q(Z)}{u^2},
\label{eq:swap-variance-G}\\
\mathcal V_+(Q,Z)
&\le \frac{6H_0^2r}{u^2}+\frac{2H_0^2Q(Z)}{u^2}.
\label{eq:swap-variance-Q}
\eal
Although $Q(Z)$ need not be nonnegative pointwise, the nonemptiness of
$\cH$ and the bound $T_n(h)\le r$ imply $Q(Z)\ge-r$.  Thus the right-hand
sides above are nonnegative.
\end{lemma}

\begin{proof}
\noindent\textit{Proof outline.}  Use an exact maximizing function first to
bound the positive swap increments of $G$ and $Q$, average the resulting
squared coordinate differences over all swaps, and then remove the attainment
assumption by an $\varepsilon$-maximizer argument.
\noindent\textit{Step 1. Bound the increments of the original supremum.}
Fix $Z\in\Omega$ and a neighboring set
$Z'=Z^{(ij)}$, where $i\in Z$ and $j\notin Z$.  First suppose that the
supremum defining $G(Z)$ is attained at $h_Z\in\cH$.  If
$G(Z)\ge G(Z')$, comparison with the same function at $Z'$ gives
$0\le G(Z)-G(Z')\le(1/u+1/m)(h_Z(i)-h_Z(j))$.
Since $u\le m$, $(1/u+1/m)^2\le4/u^2$.  The terms with
$G(Z)<G(Z')$ vanish in $\mathcal V_+(G,Z)$.  Enlarging the
remaining sum to all ordered swaps and using
$(a-b)^2\le2a^2+2b^2$, we obtain
\bals
\mathcal V_+(G,Z)
&\le\frac{8}{u^2}\sum_{i\in Z}\sum_{j\notin Z}
\frac{h_Z^2(i)+h_Z^2(j)}{um}\\
&=\frac{8}{u^3m}
\left(unT_n(h_Z)+(m-u)\sum_{i\in Z}h_Z^2(i)\right)\\
&\le\frac{16T_n(h_Z)}{u^2}
+\frac{8}{u^2}\cL_{h_Z^2}(Z)\\
&=\frac{24T_n(h_Z)}{u^2}
+\frac{8}{u^2}\bigl(\cL_{h_Z^2}(Z)-T_n(h_Z)\bigr)\\
&\le\frac{24r}{u^2}+\frac{8Q(Z)}{u^2}.
\eals
Here $m\sum_{i\in Z}h_Z^2(i)+u\sum_{j\notin Z}h_Z^2(j)
=unT_n(h_Z)+(m-u)\sum_{i\in Z}h_Z^2(i)$, and the next inequality uses
$n/m=1+u/m\le2$ and
$(m-u)/m\le1$.
This proves (\ref{eq:swap-variance-G}) when the supremum is attained.

\noindent\textit{Step 2. Bound the increments of the squared process.}
Choose $h_Z\in\cH$ so that $h_Z^2$ attains the
supremum defining $Q(Z)$.  If $Q(Z)\ge Q(Z')$, then
$0\le Q(Z)-Q(Z')\le u^{-1}\bigl(h_Z^2(i)-h_Z^2(j)\bigr)$.
Using $0\le h_Z^2(k)\le H_0^2$ and therefore
$h_Z^4(k)\le H_0^2h_Z^2(k)$, the same enlargement and calculation give
\bals
\mathcal V_+(Q,Z)
&\le\frac{2}{u^2}\sum_{i\in Z}\sum_{j\notin Z}
\frac{h_Z^4(i)+h_Z^4(j)}{um}\\
&=\frac{2}{u^3m}
\left(un\cL_n(h_Z^4)+(m-u)\sum_{i\in Z}h_Z^4(i)\right)\\
&\le\frac{4\cL_n(h_Z^4)}{u^2}
+\frac{2\cL_{h_Z^4}(Z)}{u^2}\\
&\le\frac{4H_0^2T_n(h_Z)}{u^2}
+\frac{2H_0^2}{u^2}\cL_{h_Z^2}(Z)\\
&=\frac{6H_0^2T_n(h_Z)}{u^2}
+\frac{2H_0^2}{u^2}
\bigl(\cL_{h_Z^2}(Z)-T_n(h_Z)\bigr)\\
&\le\frac{6H_0^2r}{u^2}+\frac{2H_0^2Q(Z)}{u^2}.
\eals
\noindent\textit{Step 3. Remove the attainment assumption.}
For $\varepsilon>0$, choose $h_{Z,\varepsilon}\in\cH$ whose objective value in
the definition of $G(Z)$ is at least $G(Z)-\varepsilon$.  The same comparison
as above gives, for every neighbor $Z^{(ij)}$,
$\bigl(G(Z)-G(Z^{(ij)})\bigr)_+
\le\left[(1/u+1/m)
\bigl(h_{Z,\varepsilon}(i)-h_{Z,\varepsilon}(j)\bigr)
+\varepsilon\right]_+$.
For every $a\in\RR$ and $\varepsilon\ge0$,
$[(a+\varepsilon)_+]^2\le a^2+2\abth{a}\varepsilon+\varepsilon^2$.
Here $\abth{a}\le4H_0/u$, because $u\le m$.  Summation against the transition
probabilities therefore adds at most
$8H_0\varepsilon/u+\varepsilon^2$ to the preceding bound for
$\mathcal V_+(G,Z)$.  The terms involving $h_{Z,\varepsilon}$ are bounded
uniformly by the right-hand side of (\ref{eq:swap-variance-G}), since
$T_n(h_{Z,\varepsilon})\le r$ and
$\cL_{h_{Z,\varepsilon}^2}(Z)-T_n(h_{Z,\varepsilon})\le Q(Z)$.
Letting $\varepsilon\downarrow0$ proves
(\ref{eq:swap-variance-G}).  Applying the same argument to an
$\varepsilon$-maximizer for $Q(Z)$ gives
$\bigl(Q(Z)-Q(Z^{(ij)})\bigr)_+
\le\bth{u^{-1}\pth{h_{Z,\varepsilon}^2(i)-h_{Z,\varepsilon}^2(j)}
+\varepsilon}_+$.
Now $\abth{a}\le H_0^2/u$, so the additional error in the transition average is
at most $2H_0^2\varepsilon/u+\varepsilon^2$.  Letting
$\varepsilon\downarrow0$ proves (\ref{eq:swap-variance-Q}).
\end{proof}

\begin{lemma}
[Bounded derivatives of a two-parameter log-MGF]
\label{lemma:two-parameter-log-mgf-bounded-derivatives}
Let $X,Y$ be bounded real-valued functions on the finite probability space
$\Omega$.  Set $A_X\defeq\sup_{Z\in\Omega}\abth{X(Z)}$ and
$A_Y\defeq\sup_{Z\in\Omega}\abth{Y(Z)}$.  For
$\Psi_{X,Y}(\alpha,\lambda)\defeq
\log\Expect{}{\exp\pth{\alpha X(Z)+\lambda Y(Z)}}$, the function
$\Psi_{X,Y}$ is infinitely differentiable on $\RR^2$ and, for every
$(\alpha,\lambda)\in\RR^2$,
\bals
\abth{\partial_\alpha\Psi_{X,Y}}\le A_X,\quad
\abth{\partial_\lambda\Psi_{X,Y}}\le A_Y,\\
\abth{\partial_{\alpha\alpha}^2\Psi_{X,Y}}\le A_X^2,\quad
\abth{\partial_{\lambda\lambda}^2\Psi_{X,Y}}\le A_Y^2,\\
\abth{\partial_{\alpha\lambda}^2\Psi_{X,Y}}\le A_XA_Y.
\eals
If $\Expect{}{X}=\Expect{}{Y}=0$, then
$\partial_\alpha\Psi_{X,Y}(0,0)=
\partial_\lambda\Psi_{X,Y}(0,0)=0$.
\end{lemma}

\begin{proof}
\noindent\textit{Proof outline.}  Differentiate the finite exponential sum
under its tilted law, identify first derivatives as tilted means and second
derivatives as tilted covariances, and use boundedness and Cauchy--Schwarz.
Let $M(\alpha,\lambda)\defeq
\Expect{}{\exp\pth{\alpha X(Z)+\lambda Y(Z)}}$. Since $\Omega$ is finite and
$M>0$, $\Psi_{X,Y}=\log M$ is smooth. Under the tilted probability
$\mathbb P_{\alpha,\lambda}(Z)\defeq
\pi(Z)\exp\pth{\alpha X(Z)+\lambda Y(Z)}/M(\alpha,\lambda)$,
we write $\Expect{\alpha,\lambda}{\cdot}$ for expectation under
$\mathbb P_{\alpha,\lambda}$. Differentiation under the finite sum and
$\Psi_{X,Y}=\log M$ give
$\partial_\alpha\Psi_{X,Y}=\Expect{\alpha,\lambda}{X}$ and
$\partial_\lambda\Psi_{X,Y}=\Expect{\alpha,\lambda}{Y}$.
Another differentiation gives
\bals
\partial_{\alpha\alpha}^2\Psi_{X,Y}
&=\mathop\mathrm{Var}_{\alpha,\lambda}(X),\\
\partial_{\lambda\lambda}^2\Psi_{X,Y}
&=\mathop\mathrm{Var}_{\alpha,\lambda}(Y),\\
\partial_{\alpha\lambda}^2\Psi_{X,Y}
&=\Cov_{\alpha,\lambda}(X,Y).
\eals
The bounds follow from $\abth{X}\le A_X$, $\abth{Y}\le A_Y$, and
$\abth{\Cov_{\alpha,\lambda}(X,Y)}
\le\sqrt{\mathop\mathrm{Var}_{\alpha,\lambda}(X)
\mathop\mathrm{Var}_{\alpha,\lambda}(Y)}\le A_XA_Y$.
The identities at $(0,0)$ follow from $\mathbb P_{0,0}=\pi$.
\end{proof}

\begin{lemma}[Characteristic curve for the two-parameter closure]
\label{lemma:two-parameter-characteristic-curve}
Fix $u\in\NN$ and $H_0>0$, and let
$\mathsf B(\alpha,\lambda)\defeq(16\alpha^2+4H_0^2\lambda^2)/u$.
For every $0<\beta\le u/(16H_0)$, the backward characteristic equation
$t\lambda_t'=\lambda_t-\mathsf B(t\beta,\lambda_t)$, with
$\lambda_1=0$ and $0<t\le1$, has the unique solution
\bal\label{eq:characteristic-explicit-solution}
\lambda_t=\frac{2\beta t}{H_0}
\tan\pth{\frac{8H_0\beta}{u}(1-t)},\qquad 0<t\le1.
\eal
This solution satisfies $0\le\lambda_t\le18\beta^2t(1-t)/u$ for
$0<t\le1$.
Moreover, it extends smoothly to $t=0$ by $\lambda_0=0$. In particular,
as $t\downarrow0$, $\lambda_t=O(t)$, $\lambda_t'=O(1)$, and
$\lambda_t''=O(1)$.
\end{lemma}

\begin{proof}
\noindent\textit{Proof outline.}  Divide the characteristic by $t$, solve the
resulting Riccati equation through an arctangent transform, and control the
explicit tangent formula uniformly up to the endpoint $t=0$.
Put $y_t=\lambda_t/t$ and $\kappa=8H_0\beta/u$.  Since
$\lambda_t=y_tt$, the characteristic equation is equivalent on $(0,1]$ to
$y_t'=-(16\beta^2+4H_0^2y_t^2)/u$ with $y_1=0$.  Equivalently,
$\frac{\diff}{\diff t}\arctan\pth{H_0y_t/(2\beta)}=-\kappa$.
Integration from $t$ to $1$ gives
$\arctan\pth{H_0y_t/(2\beta)}=\kappa(1-t)$.
Because $0\le\kappa(1-t)\le1/2<\pi/2$, this identity uniquely determines
$y_t$ and proves (\ref{eq:characteristic-explicit-solution}).

It remains to verify the stated comparison constant.  We claim that
$\tan z\le9z/8$ for $0\le z\le1/2$.
Indeed, $f(z)=9z/8-\tan z$ is concave on this interval, since
$f''(z)=-2\sec^2(z)\tan(z)\le0$.  Moreover, $f(0)=0$.  At $z=1/2$, the
alternating-series remainder theorem gives
$\sin z\le z-z^3/6+z^5/120$ and
$\cos z\ge1-z^2/2>0$.  Therefore
$\tan(1/2)\le
(1/2-1/48+1/3840)/(1-1/8)=1841/3360<9/16$.
Thus $f(1/2)>0$, and concavity places $f$ above the chord joining its endpoint
values, proving the claim.  Applying it with
$z=\kappa(1-t)$ in (\ref{eq:characteristic-explicit-solution}) yields
$0\le\lambda_t\le(2\beta t/H_0)(9/8)(8H_0\beta/u)(1-t)
=18\beta^2t(1-t)/u$.
Finally, the explicit formula is smooth for $0\le t\le1$, because its
tangent argument remains in $[0,1/2]$.  The asserted endpoint estimates
follow immediately.
\end{proof}

\begin{proposition}[Two-parameter Herbst closure]
\label{proposition:two-parameter-herbst-closure}
Fix $u\in\NN$ and $H_0,S>0$.  Let $F,G$ be centered bounded random
variables on the finite probability space $(\Omega,\pi)$, and define
$\Psi(\alpha,\lambda)\defeq
\log\Expect{}{\exp\pth{\alpha F+\lambda G}}$ for $\alpha,\lambda\ge0$.
Put $\mathsf B(\alpha,\lambda)\defeq
(16\alpha^2+4H_0^2\lambda^2)/u$.
Assume that, for all $\alpha,\lambda\ge0$,
\bals
\alpha\frac{\partial\Psi(\alpha,\lambda)}{\partial \alpha}
+\pth{\lambda-\mathsf B(\alpha,\lambda)}
\frac{\partial\Psi(\alpha,\lambda)}{\partial \lambda}
-\Psi(\alpha,\lambda)
\le
\frac{48S\alpha^2+16H_0^2S\lambda^2}{u}.
\eals
Then $\Psi(\beta,0)\le219S\beta^2/(4u)$ for every
$0<\beta\le u/(16H_0)$.
\end{proposition}

\begin{proof}
\noindent\textit{Proof outline.}  Evaluate the log-MGF along the backward
characteristic of Lemma~\ref{lemma:two-parameter-characteristic-curve}, use
the assumed differential inequality to bound its derivative, and integrate
from the centered endpoint at $t=0$ to $t=1$.
The characteristic is chosen to match the coefficient of
$\partial_\lambda\Psi$ in the assumed differential inequality.
Let $\lambda_t$ be the curve in
Lemma~\ref{lemma:two-parameter-characteristic-curve}, and set
$\phi(t)\defeq\Psi(t\beta,\lambda_t)/t$ for $0<t\le1$.
The endpoint relation $\lambda_1=0$ gives $\phi(1)=\Psi(\beta,0)$.
Let $H(t)=\Psi(t\beta,\lambda_t)$ for $0\le t\le1$.  The characteristic
extends smoothly to $t=0$ with $\lambda_0=0$, and the log-MGF is smooth by
Lemma~\ref{lemma:two-parameter-log-mgf-bounded-derivatives}. Hence
$H\in C^2([0,1])$.  Since $F$ and $G$ are centered,
$\partial_\alpha\Psi(0,0)=\partial_\lambda\Psi(0,0)=0$.  Therefore
$H(0)=0$ and
$H'(0)=\beta\partial_\alpha\Psi(0,0)
+\lambda_0'\partial_\lambda\Psi(0,0)=0$.
Taylor's theorem at zero now gives $H(t)=O(t^2)$, and consequently
$\lim_{t\downarrow0}\phi(t)=0$.
Differentiating $\Psi(t\beta,\lambda_t)/t$ and using the characteristic equation gives
\bals
\phi'(t)
=
\frac{
t\beta\frac{\partial\Psi(t\beta,\lambda_t)}{\partial \alpha}
+t\lambda'_t\frac{\partial\Psi(t\beta,\lambda_t)}{\partial \lambda}
-\Psi(t\beta,\lambda_t)}
{t^2}.
\eals
Substituting $t\lambda_t'=\lambda_t-\mathsf B(t\beta,\lambda_t)$
makes the numerator exactly the left-hand side of the assumed differential
inequality at $(t\beta,\lambda_t)$.  It follows that
$\phi'(t)\le48S\beta^2/u+16H_0^2S\lambda_t^2/(ut^2)$.
Integrating over $t\in(0,1)$ and using
$\lambda_t\le18\beta^2t(1-t)/u$ gives
\bals
\Psi(\beta,0)
\le
\frac{48S\beta^2}{u}
+\frac{16H_0^2S}{u}\int_0^1
\pth{\frac{18\beta^2(1-t)}{u}}^2\diff t
\le \frac{219S\beta^2}{4u},
\eals
where, more explicitly, the integral term equals
$1728H_0^2S\beta^4/u^3\le27S\beta^2/(4u)$.
The last inequality uses $\beta\le u/(16H_0)$, and
$48+27/4=219/4$.
\end{proof}

\begin{proof}[\textbf{\textup{Proof of Theorem~\ref{theorem:concentration-gd-STLC}}}]
\noindent\textit{Proof outline.}  For $u\le m$, we first center the
original supremum and its squared auxiliary process.  We then derive their
joint entropy inequality, close it along the characteristic, and convert
the resulting moment bound into a tail bound.  Finally, complementation
handles $m<u$.

\noindent\textit{Step 1. Identify the two processes to be controlled.}
Assume $u\le m$.  Recall that $g(\cH,Z)$ is defined in
(\ref{eq:def-g}) and that
$g_u^+(\cH^2,Z)=\sup_{h\in\cH}\set{\cL_u(h^2)-T_n(h)}$.  Put
$g_2(Z)\defeq g_u^+(\cH^2,Z)$ and
$R_2\defeq\Expect{}{g_2(Z)}=\cfrakR_u^+(\cH^2)$.  Notice that
$R_2\ge0$, because the expectation of the supremum is at least the
supremum of the expectations, each of which is zero.  The variable
$g_2(Z)$ itself need not be nonnegative.  Its expectation is the
square-class complexity that appears in the theorem.
Let $F(Z)\defeq g(\cH,Z)-\Expect{}{g(\cH,Z)}$,
$G_2(Z)\defeq g_2(Z)-R_2$, and $S\defeq r+R_2>0$.  For
$\alpha,\lambda\ge0$, define
$\Psi(\alpha,\lambda)\defeq
\log \Expect{}{\exp\pth{\alpha F(Z)+\lambda G_2(Z)}}$.

\noindent\textit{Step 2. Derive the joint entropy inequality.}
For a neighboring pair $Z,Z'$ with $P(Z,Z')>0$, write
$\Delta a=a(Z)-a(Z')$ for a function $a$ on $\Omega$.
For $\alpha,\lambda\ge0$, the elementary
inequalities $(a+b)_+\le a_++b_+$ and $(a+b)^2\le2a^2+2b^2$ give
$\pth{\alpha\Delta g+\lambda\Delta g_2}_+^2
\le2\alpha^2(\Delta g)_+^2+2\lambda^2(\Delta g_2)_+^2$,
and centering does not change the increments. Therefore,
\bal\label{eq:main-inequality-sup-process-general-two-param-positive-variance}
\mathcal V_+(\alpha F+\lambda G_2,Z)
&=\sum_{Z'\in\Omega}P(Z,Z')
\pth{\alpha\Delta g+\lambda\Delta g_2}_+^2\nonumber\\
&\le2\alpha^2\mathcal V_+(g,Z)
+2\lambda^2\mathcal V_+(g_2,Z).
\eal
Using (\ref{eq:main-inequality-sup-process-general-two-param-positive-variance}), the modified log-Sobolev inequality
(\ref{eq:mlsi-swap}) and Lemma~\ref{lemma:dirichlet-upper-bound}, applied
with $q=\alpha F+\lambda G_2$ and $\gamma=1$, give
\bal\label{eq:main-inequality-sup-process-general-two-param-entropy}
\frac{\Ent_\pi\pth{\exp\pth{\alpha F+\lambda G_2}}}
{\Expect{}{\exp\pth{\alpha F+\lambda G_2}}} 
&\le
u\frac{\Expect{}{\mathcal V_+(\alpha F+\lambda G_2,Z)
\exp\pth{\alpha F+\lambda G_2}}}
{\Expect{}{\exp\pth{\alpha F+\lambda G_2}}} \nonumber\\
&\le
u\frac{\Expect{}{\pth{2\alpha^2\mathcal V_+(g,Z)+2\lambda^2\mathcal V_+(g_2,Z)}
\exp\pth{\alpha F+\lambda G_2}}}
{\Expect{}{\exp\pth{\alpha F+\lambda G_2}}}.
\eal
Combining (\ref{eq:main-inequality-sup-process-general-two-param-entropy}) with
Lemma~\ref{lemma:oriented-swap-variance-bounds} and using
$\Expect{}{g_2(Z)\exp\pth{\alpha F+\lambda G_2}}/
\Expect{}{\exp\pth{\alpha F+\lambda G_2}}
=R_2+\partial_\lambda\Psi(\alpha,\lambda)$,
we obtain
\bal\label{eq:main-inequality-sup-process-general-two-param-pde-raw}
&\frac{\Ent_\pi\pth{\exp\pth{\alpha F+\lambda G_2}}}
{\Expect{}{\exp\pth{\alpha F+\lambda G_2}}} \le
\frac{48r\alpha^2+12H_0^2r\lambda^2}{u}
+
\frac{16\alpha^2+4H_0^2\lambda^2}{u}
\pth{R_2+\frac{\partial \Psi(\alpha,\lambda)}{\partial \lambda}}.
\eal
On the other hand, let $\mathbb P_{\alpha,\lambda}$ denote the exponential
tilt of $\pi$ by $\alpha F+\lambda G_2$.  Differentiating the finite sum shows
that $\partial_\alpha\Psi$ and $\partial_\lambda\Psi$ are the corresponding
tilted expectations of $F$ and $G_2$, respectively.  Consequently,
\bal\label{eq:main-inequality-sup-process-general-two-param-entropy-identity}
\frac{\Ent_\pi\pth{\exp\pth{\alpha F+\lambda G_2}}}
{\Expect{}{\exp\pth{\alpha F+\lambda G_2}}}
=
\alpha\frac{\partial\Psi(\alpha,\lambda)}{\partial \alpha}
+\lambda\frac{\partial\Psi(\alpha,\lambda)}{\partial \lambda}
-\Psi(\alpha,\lambda).
\eal
Thus, with $\mathsf B(\alpha,\lambda)\defeq
(16\alpha^2+4H_0^2\lambda^2)/u$,
(\ref{eq:main-inequality-sup-process-general-two-param-pde-raw}) and
(\ref{eq:main-inequality-sup-process-general-two-param-entropy-identity}) imply
\bal\label{eq:main-inequality-sup-process-general-two-param-pde}
&\alpha\frac{\partial\Psi(\alpha,\lambda)}{\partial \alpha}
+\pth{\lambda-\mathsf B(\alpha,\lambda)}
\frac{\partial\Psi(\alpha,\lambda)}{\partial \lambda}
-\Psi(\alpha,\lambda) \nonumber\\
&\le
\frac{\pth{48r+16R_2}\alpha^2+\pth{12H_0^2r+4H_0^2R_2}\lambda^2}{u}
\le
\frac{48S\alpha^2+16H_0^2S\lambda^2}{u}.
\eal
\noindent\textit{Step 3. Close the moment bound.}
We verify the boundedness and centering hypotheses of
Proposition~\ref{proposition:two-parameter-herbst-closure}.  Since
$\abth{h(i)}\le H_0$, we have $-2H_0\le g(\cH,Z)\le2H_0$, and hence
$\abth{F(Z)}\le4H_0$.  Similarly, $0\le h^2(i)\le H_0^2$ gives
$-H_0^2\le g_2(Z)\le H_0^2$, and hence
$\abth{G_2(Z)}\le2H_0^2$.  We therefore take $A_X=4H_0$ and
$A_Y=2H_0^2$ in
Lemma~\ref{lemma:two-parameter-log-mgf-bounded-derivatives}.
It follows that $\Psi$ is twice continuously differentiable with bounded first and second derivatives in a neighborhood of $(0,0)$. Moreover, since $F$ and $G_2$ are centered,
$\partial_\alpha\Psi(0,0)=\Expect{}{F}=0$ and
$\partial_\lambda\Psi(0,0)=\Expect{}{G_2}=0$.
Therefore, Proposition~\ref{proposition:two-parameter-herbst-closure} applied to (\ref{eq:main-inequality-sup-process-general-two-param-pde}) yields, for every $0<\beta\le u/(16H_0)$,
\bal\label{eq:main-inequality-sup-process-general-log-mgf}
\log \Expect{}{\exp\pth{\beta(g(\cH,Z)-\Expect{}{g(\cH,Z)})}}
=\Psi(\beta,0)
\le \frac{219S\beta^2}{4u}.
\eal
\noindent\textit{Step 4. Convert the moment bound into a tail bound.}
Chernoff's method applied to
(\ref{eq:main-inequality-sup-process-general-log-mgf}) gives, with
probability at least $1-\exp(-x)$,
\bal\label{eq:main-inequality-sup-process-general-u-le-m-final}
g(\cH,Z)-\Expect{}{g(\cH,Z)}
&\le c_0\sqrt{\frac{(r+R_2)x}{u}}+\frac{16H_0x}{u}.
\eal
Indeed, set $a=219S/(4u)$ and $b=u/(16H_0)$.  Chernoff's inequality and
(\ref{eq:main-inequality-sup-process-general-log-mgf}) give
$\Prob{F\ge t}\le\exp\set{-\beta t+a\beta^2}$ for $0<\beta\le b$.
Take $t=2\sqrt{ax}+x/b$ and
$\beta=\min\set{\sqrt{x/a},b}$.  If $\sqrt{x/a}\le b$, then
$-\beta t+a\beta^2=-x-(x/b)\sqrt{x/a}\le-x$.
If $\sqrt{x/a}>b$, then $x>ab^2$ and
$-\beta t+a\beta^2=-x-2b\sqrt{ax}+ab^2\le-x$.
This proves (\ref{eq:main-inequality-sup-process-general-u-le-m-final}).
Since $c_0=\sqrt{219}$, the inequalities
$\sqrt{a+b}\le\sqrt a+\sqrt b$ and $2\sqrt{ab}\le a+b$ for $a,b\ge0$
give $c_0\sqrt{(r+R_2)x/u}\le c_0\sqrt{rx/u}
+(c_0/2)R_2+c_0x/(2u)$.  Consequently,
$g(\cH,Z)-\Expect{}{g(\cH,Z)}
\le c_0\sqrt{rx/u}+(c_0/2)R_2+(16H_0+c_0/2)x/u$.
This proves (\ref{eq:concentration-gd-STLC}) when $u\le m$.

\noindent\textit{Step 5. Exchange the two subsets when $m<u$.}
Apply the preceding argument to $-\cH$ and to $\bar Z$, which is uniformly
distributed over the subsets of size $m$.
The class $-\cH$ has the same envelope $H_0$ and the same second-moment
radius $r$.  The process is unchanged because
$g(-\cH,\barZ)=\sup_{h\in\cH}\pth{\cL_{-h}(\barZ)-\cL_{-h}(Z)}
=\sup_{h\in\cH}\pth{\cL_h(Z)-\cL_h(\barZ)}=g(\cH,Z)$.
Moreover, $(-\cH)^2=\cH^2$ and $T_n(-h)=T_n(h)$.  The auxiliary complexity
in the transformed experiment is
$\Expect{}{\sup_{h\in\cH}\pth{\cL_{h^2}(\bar Z)-T_n(h)}}
=\cfrakR_m^+(\cH^2)$.
Hence the same bound holds with $u$ replaced by $m$ and with the
squared-class complexity $\cfrakR_m^+(\cH^2)$.  Combining the two cases gives
the asserted bound with $N_{u,m}=\min\set{u,m}$ and
$\cfrakR_{N_{u,m}}^+(\cH^2)$.  When $u=m$, the map
$Z\mapsto\bar Z$ preserves the uniform law and interchanges the two empirical
averages, so the two complexity values agree and the notation is unambiguous.
\end{proof}

\subsection{Proofs of Theorem~\ref{theorem:TLC}}
\label{sec:proofs-theorem-STLC-STLC-nonnegative-func-class}

The rescaling below places every function in a class with a common
second-moment bound.  The deterministic lemma then converts a bound for
this rescaled class into the desired bound for the original functions.

Fix $r>0$ and $\lambda>1$.  For $h\in\cH$, define
$w(h)\defeq\min\set{r\lambda^k\colon k\in\NN\cup\set{0},
r\lambda^k\ge\tT_n(h)}$ and
$\cH^{(r)}\defeq\set{rh/w(h)\colon h\in\cH}$.
The minimum exists because $\tT_n(h)$ is finite and $r\lambda^k\to\infty$.
When $\tT_n(h)\le r$, the rescaling leaves $h$ unchanged.
When $\tT_n(h)>r$, it uses the smallest geometric radius above $\tT_n(h)$.
Define
\bal\label{eq:Ur+}
U_r^+\defeq g(\cH^{(r)},Z)
=\sup_{s\in\cH^{(r)}}\set{\cL_u(s)-\cL_m(s)}.
\eal
The following implication holds for every split $Z$.

\begin{lemma}\label{lemma:STLC-supp-lemma}
Fix $\lambda > 1$, $K_0 > 1$, and $r > 0$. If $U_r^+ \le \frac{r}{\lambda K_0 }$, then
\bal\label{eq:STLC-supp-lemma}
\cL_h(Z) \le \cL_h(\barZ)+ \frac{r}{\lambda K_0 } + \frac{\tT_n(h)}{K_0}, \quad \forall h \in \cH.
\eal
\end{lemma}
\begin{proof}
\noindent\textit{Proof outline.}  Separate the cases below and above the
localization radius $r$, undo the deterministic gauge rescaling in each case,
and compare the selected shell radius with $\tT_n(h)$.
If $\tT_n(h) \le r$, then $w(h) = r$ and $s = \frac{r}{w(h)} h = h$. Therefore,
$U_r^+ \le \frac{r}{\lambda K_0} \Rightarrow \cL_u(s)-\cL_m(s) \le \frac{r}{\lambda K_0}$ and (\ref{eq:STLC-supp-lemma}) holds since $\tT_n(h) \ge 0$ for all $h \in \cH$.

If $\tT_n(h) > r$, then $w(h) = r\lambda^k$ with $\tT_n(h) \in (r\lambda^{k-1},r\lambda^k]$.
Again, it follows from $U_r^+ \le \frac{r}{\lambda K_0}$ that
$\cL_u(s)-\cL_m(s)\le r/(\lambda K_0)$ for $s=h/\lambda^k$.
Consequently,
$\cL_h(Z)-\cL_h(\barZ)\le r\lambda^{k-1}/K_0
\le\tT_n(h)/K_0$,
and (\ref{eq:STLC-supp-lemma}) still holds.
\end{proof}

\begin{proof}[\textbf{\textup{Proof of Theorem~\ref{theorem:TLC}}}]
\noindent\textit{Proof outline.}  We verify the second-moment bound for the
rescaled class and apply concentration at a fixed radius.  We then bound
the expected deviations of this class and its square by geometric peeling.
Finally, we choose the radius deterministically and undo the rescaling.

\noindent\textit{Step 1. Apply concentration to one rescaled class.}
Fix a deterministic $r\ge r_{u,m}$ and $\lambda>1$, and recall $w(h)$,
$\cH^{(r)}$, and $U_r^+$ from (\ref{eq:Ur+}).
For $s=rh/w(h)\in\cH^{(r)}$, we have
$\abth{s(i)}\le\abth{h(i)}\le H_0$ and
\bals
T_n(s)=\pth{\frac{r}{w(h)}}^2T_n(h)
\stackrel{\circled{1}}{\le}
\pth{\frac{r}{w(h)}}^2\tT_n(h)
\le\frac{r^2}{w(h)}\le r.
\eals
Step \circled{1} uses $T_n(h)\le\tT_n(h)$.  The remaining inequalities
use $\tT_n(h)\le w(h)$ and $w(h)\ge r$.
The rescaling is justified by the homogeneity of $T_n$ alone.
Thus Theorem~\ref{theorem:concentration-gd-STLC} applies to $\cH^{(r)}$.
For any nonempty bounded class $\cA$ and $a\in\cA$, the identity
$\cL_u(a)-\cL_m(a)=R_u^+a+R_m^-a$ and subadditivity of the supremum give
$\Expect{}{g(\cA,Z)}\le\cfrakR_u^+(\cA)+\cfrakR_m^-(\cA)$.
Theorem~\ref{theorem:concentration-gd-STLC} therefore implies that, with
probability at least $1-\exp(-x)$,
\bal\label{eq:STLC-seg-U+}
U_r^+ &\le \cfrakR_u^+\pth{\cH^{(r)}}
+\cfrakR_m^-\pth{\cH^{(r)}}
+c_0\sqrt{\frac{rx}{N_{u,m}}}
+\frac{c_0}2
\cfrakR_{N_{u,m}}^+\pth{(\cH^{(r)})^2}
+\frac{(16H_0+c_0/2)x}{N_{u,m}}.
\eal
\noindent\textit{Step 2. Bound the expectations by peeling.}
The following estimates are deterministic.  They do not require separate
concentration events for the individual shells.
Put $\cS_0\defeq\set{h\in\cH\colon\tT_n(h)\le r}$ and
$\cS_k\defeq\set{h\in\cH\colon r\lambda^{k-1}<\tT_n(h)\le r\lambda^k}$
for $k\ge1$.  The rescaled functions on shell $\cS_k$ are $\lambda^{-k}h$.
Enlarge each shell to $\cH(r\lambda^k)$, whose centered supremum is
nonnegative because this localized class contains zero.  By the linearity
of $R_p^\eta$, for every $p\in\set{u,m}$ and $\eta\in\set{+,-}$, pointwise,
\bals
\sup_{s\in\cH^{(r)}}R_p^\eta s
\stackrel{\circled{2}}{\le}\sum_{k=0}^{\infty}\lambda^{-k}
\sup_{h\in\cH(r\lambda^k)}R_p^\eta h.
\eals
In step \circled{2}, each rescaled function belongs to one shell, and that
shell's supremum is bounded by the sum of the nonnegative enlarged
suprema.
The same argument for the square class uses the rescaling factor
$\lambda^{-2k}$ and the enlarged class $\cH^2(r\lambda^k)$.
Taking expectations and applying monotone convergence to the nonnegative
series gives
\bal
\cfrakR_p^\eta\pth{\cH^{(r)}}
&\le \cfrakR_p^\eta\pth{\cH(r)}
+\sum_{k=1}^{\infty}\lambda^{-k}
\cfrakR_p^\eta\pth{\cH(r\lambda^k)},
\label{eq:STLC-scaled-TC-peeling}\\
\cfrakR_p^\eta\pth{(\cH^{(r)})^2}
&\le \cfrakR_p^\eta\pth{\cH^2(r)}
+\sum_{k=1}^{\infty}\lambda^{-2k}
\cfrakR_p^\eta\pth{\cH^2(r\lambda^k)}.
\label{eq:STLC-scaled-square-TC-peeling}
\eal
Empty shells contribute nothing before enlargement.  The zero functions
in the enlarged classes ensure that every summand is nonnegative, which
justifies the use of monotone convergence even when there are infinitely
many nonempty shells.
For $a\ge1$ and the fixed $r\ge r_{u,m}$, the sub-root property gives
$\psi_{u,m}(ar)\le\sqrt a\,\psi_{u,m}(r)$.
With $\lambda=4$, the coefficients in
(\ref{eq:STLC-scaled-TC-peeling}) sum to
$\sum_{k\ge0}4^{-k/2}=2$, whereas those in
(\ref{eq:STLC-scaled-square-TC-peeling}) sum to
$\sum_{k\ge0}4^{-3k/2}=8/7$.
Condition (\ref{eq:STLC-cond-psi-general}) therefore gives
\bal\label{eq:STLC-scaled-TC-bound}
\max_{\substack{p\in\set{u,m}\\ \eta\in\set{+,-}}}
\cfrakR_p^\eta\pth{\cH^{(r)}}
&\le2\psi_{u,m}(r),\\
\max_{\substack{p\in\set{u,m}\\ \eta\in\set{+,-}}}
\cfrakR_p^\eta\pth{(\cH^{(r)})^2}
&\le\frac87\psi_{u,m}(r).
\label{eq:STLC-scaled-square-TC-bound}
\eal
Since $\psi_{u,m}$ is sub-root and
$\psi_{u,m}(r_{u,m})=r_{u,m}$,
$\psi_{u,m}(r)\le\sqrt{rr_{u,m}}$ for $r\ge r_{u,m}$.
It follows from (\ref{eq:STLC-seg-U+}),
(\ref{eq:STLC-scaled-TC-bound}), and
(\ref{eq:STLC-scaled-square-TC-bound}) that
\bal\label{eq:STLC-seg3}
U_r^+ &\le d_0\sqrt{r r_{u,m}}
+c_0\sqrt{\frac{rx}{N_{u,m}}}
+\frac{(16H_0+c_0/2)x}{N_{u,m}}
\defeq P(r),
\eal
where $d_0=4+4c_0/7$.

\noindent\textit{Step 3. Choose the radius and undo the rescaling.}
Recall that $P(r)$ is the upper bound in (\ref{eq:STLC-seg3}).
For the prescribed $x>0$, define
$a\defeq d_0\sqrt{r_{u,m}}+c_0\sqrt{x/N_{u,m}}$,
$b\defeq(16H_0+c_0/2)x/N_{u,m}$, and
$r_1\defeq\lambda^2K_0^2a^2+2\lambda K_0b$.
These quantities are deterministic because the full sample and the
complexity majorant are fixed.
Since $\lambda=4$, $K_0>1$, and $d_0>1$, we have $r_1\ge r_{u,m}$.
We apply the concentration event above at this single radius $r_1$.
The inequality $2st\le s^2+t^2$, which follows from Cauchy--Schwarz, gives
$P(r_1)=a\sqrt{r_1}+b
\le r_1/(2\lambda K_0)+\lambda K_0a^2/2+b
=r_1/(\lambda K_0)$.
Therefore, setting $r=r_1$ in (\ref{eq:STLC-seg3}) gives
$U_{r_1}^+\le r_1/(\lambda K_0)=
\lambda K_0\pth{d_0\sqrt{r_{u,m}}+c_0\sqrt{x/N_{u,m}}}^2
+(32H_0+c_0)x/N_{u,m}$.
It follows from Lemma~\ref{lemma:STLC-supp-lemma} that
\bals
\cL_h(Z) &\le \cL_h(\barZ) + \frac{r_1}{\lambda K_0}
+ \frac{\tT_n(h)}{K_0} \\
&=\cL_h(\barZ)+ \frac{\tT_n(h)}{K_0}+
\lambda K_0
\pth{d_0\sqrt{r_{u,m}}+c_0\sqrt{\frac{x}{N_{u,m}}}}^2
+ \frac {(32H_0+c_0)x}{N_{u,m}} \\
&\le \cL_h(\barZ) + \frac{\tT_n(h)}{K_0}
+c_1 r_{u,m}+\frac{c_2x}{N_{u,m}},\qquad h\in\cH.
\eals
Indeed, since $\lambda=4$ and $(s+t)^2\le2s^2+2t^2$, the two remainder
terms in the penultimate line are at most
$8K_0d_0^2r_{u,m}+\pth{8K_0c_0^2+32H_0+c_0}x/N_{u,m}
=c_1r_{u,m}+c_2x/N_{u,m}$,
which proves (\ref{eq:STLC-bound-g-upper-bound}).

\end{proof}

\subsection{Proof of Theorem~\ref{theorem:STLC-delta-ell-f-excess-risk-upper-bound}}
\label{sec:proof-generic-STLC}

Recall that $g_f=\ell_f-\ell_{f_n^*}$,
$\Delta_{\cF}^*=\set{g_f\colon f\in\cF}$, and
$\Delta_{\cF}=\set{\ell_{f_1}-\ell_{f_2}\colon f_1,f_2\in\cF}$.
The proof needs two consequences of Theorem~\ref{theorem:TLC}.
The first controls full-sample excess losses through the surrogate
$B\cL_n(h)$ on $\Delta_{\cF}^*$.
The second uses the representation infimum (\ref{eq:tTn-def}) on
$\Delta_{\cF}$.  These are different functionals on different classes.
Throughout this subsection, retain the sampling setup of
Section~\ref{sec:transductive-basic-setup} and the bounded-loss condition
$0\le\ell_f(i)\le L_0$ for every $f\in\cF$ and $i\in[n]$.
\begin{lemma}\label{lemma:concentration-gu-gm-Delta*}
Suppose Assumption~\ref{assumption:main} holds. Let $\psi^*_{u,m}$ be a
sub-root function with positive fixed point $r^*$ that satisfies
(\ref{eq:STLC-cond-um-delta-star-ell-f-psi-star}).
Then, for every fixed constant $K_0>1$, there exists a positive constant
$\hat c_{\Delta}$ depending only on $K_0$ and $L_0$ such that, for every
$x>0$, with probability at least $1-2\exp(-x)$ over $Z$, the following two
inequalities hold simultaneously for every $h\in\Delta^*_{\cF}$:
\bal
\cL_n(h)-\cL_m(h)&\le\frac{B\cL_n(h)}{K_0}+\hat c_{\Delta}\pth{r^*+\frac{x}{N_{u,m}}},
\label{eq:concentration-gm-Delta*-fixedpoint}\\
\cL_n(h)-\cL_u(h)&\le\frac{B\cL_n(h)}{K_0}+\hat c_{\Delta}\pth{r^*+\frac{x}{N_{u,m}}}.
\label{eq:concentration-gu--Delta*-fixedpoint}
\eal
\end{lemma}
\begin{proof}
\noindent\textit{Proof outline.}  Apply Theorem~\ref{theorem:TLC} to the
excess-loss class in each orientation of the random split, convert the two
test--train differences into full-sample deviations, and intersect the events.
Let $\cH=\Delta^*_{\cF}$ and, for this application only, use
$\tT_n(h)=B\cL_n(h)$.
Full-sample optimality gives $\cL_n(h)\ge0$, the bounded loss gives
$\abth{h(i)}\le L_0$, and Assumption~\ref{assumption:main}(2) gives
$T_n(h)\le\tT_n(h)$.
The classes in (\ref{eq:STLC-cond-psi-general}) are then precisely
$\cF_r^*$ and $(\cF_r^*)^2$, so the condition on $\psi^*_{u,m}$ supplies
the required complexity bound.
Theorem~\ref{theorem:TLC}, with $H_0=L_0$, gives with probability at least
$1-\exp(-x)$, simultaneously for every $h\in\cH$,
$\cL_u(h)-\cL_m(h)\le B\cL_n(h)/K_0+c_1r^*+c_2x/N_{u,m}$.
Set $\hat c_{\Delta}\defeq\max\set{c_1,c_2}$.
To obtain the first full-sample deviation, use
$\cL_n=(u/n)\cL_u+(m/n)\cL_m$ and the nonnegativity of the upper bound:
\bals
\cL_n(h)-\cL_m(h)
&=\frac{u}{n}\pth{\cL_u(h)-\cL_m(h)}\\
&\le\frac{u}{n}\pth{\frac{B\cL_n(h)}{K_0}
 +c_1r^*+\frac{c_2x}{N_{u,m}}}\\
&\le\frac{B\cL_n(h)}{K_0}+\hat c_{\Delta}\pth{r^*+\frac{x}{N_{u,m}}}.
\eals
Apply Theorem~\ref{theorem:TLC} once more after exchanging $u$ with $m$ and $Z$ with $\barZ$. The complexity condition is unchanged because its maximum ranges over both $p\in\set{u,m}$. Thus, with probability at least $1-\exp(-x)$, simultaneously for all $h\in\cH$,
$\cL_m(h)-\cL_u(h)\le B\cL_n(h)/K_0+c_1r^*+c_2x/N_{u,m}$.
On this event,
\bals
\cL_n(h)-\cL_u(h)
&=\frac{m}{n}\pth{\cL_m(h)-\cL_u(h)}\\
&\le\frac{m}{n}\pth{\frac{B\cL_n(h)}{K_0}
 +c_1r^*+\frac{c_2x}{N_{u,m}}}\\
&\le\frac{B\cL_n(h)}{K_0}+\hat c_{\Delta}\pth{r^*+\frac{x}{N_{u,m}}}.
\eals
A union bound shows that both inequalities hold with probability at least $1-2\exp(-x)$. By Theorem~\ref{theorem:TLC}, $\hat c_{\Delta}$ depends only on $K_0$ and $L_0$.
\end{proof}

We next verify that the representation infimum in (\ref{eq:tTn-def})
provides the required second-moment bound on the larger class
$\Delta_{\cF}$.

\begin{lemma}\label{lemma:STLC-delta-ell-f}
Suppose that Assumption~\ref{assumption:main} holds, and let $\tT_n$ be
defined by (\ref{eq:tTn-def}).  Then
$\tT_n\colon\Delta_{\cF}\to\RR^+$ is finite and satisfies
$T_n(h)\le\tT_n(h)$ for every $h\in\Delta_{\cF}$.  Moreover,
\bal\label{eq:tTn-excess-loss-upper-bound}
\tT_n(\ell_f-\ell_{f_n^*})
\le 2B\cL_n(\ell_f-\ell_{f_n^*}),
\qquad f\in\cF.
\eal
Consequently, the envelope and variance-majorization hypotheses of
Theorem~\ref{theorem:TLC} hold with $\cH=\Delta_{\cF}$, $H_0=L_0$, and
the surrogate functional $\tT_n$ defined in (\ref{eq:tTn-def}). Whenever
the localized-complexity condition (\ref{eq:STLC-cond-psi-general}) also
holds, the conclusion of that theorem follows.
\end{lemma}
\begin{proof}
\noindent\textit{Proof outline.}  Establish that the infimum defining
$\tT_n$ is finite and nonnegative, majorize $T_n(h)$ for every representation
of $h$ by the two Bernstein bounds, and specialize to the representation
$(f,f_n^*)$.
Recall that $g_f=\ell_f-\ell_{f_n^*}$.
By the optimality of $f_n^*$ on the full sample,
$\cL_n(g_f)\ge0$ for every $f\in\cF$.  Because each
$h\in\Delta_{\cF}$ has at least one representation
$h=\ell_{f_1}-\ell_{f_2}$, the infimum in (\ref{eq:tTn-def}) is over a
nonempty set of finite nonnegative numbers.  Thus $\tT_n(h)$ is finite and
nonnegative.

Fix any such representation and write $g_j\defeq g_{f_j}$ for
$j\in\set{1,2}$.  Since $h=g_1-g_2$, the elementary inequality
$(a-b)^2\le2a^2+2b^2$ and Assumption~\ref{assumption:main}(2) give
$T_n(h)=n^{-1}\sum_{i=1}^n\pth{g_1(i)-g_2(i)}^2
\le2T_n(g_1)+2T_n(g_2)
\le2B\set{\cL_n(g_1)+\cL_n(g_2)}$.
This inequality holds for every representation of the same $h$, so taking
the infimum proves $T_n(h)\le\tT_n(h)$.
No minimizer of the representation infimum is required.
For $h=\ell_f-\ell_{f_n^*}$, the particular
representation $(f,f_n^*)$ and $g_{f_n^*}=0$ prove
(\ref{eq:tTn-excess-loss-upper-bound}).  Finally,
$0\le\ell_f(i)\le L_0$ implies
$\abth{\ell_{f_1}(i)-\ell_{f_2}(i)}\le L_0$, so the envelope required by
Theorem~\ref{theorem:TLC} is $H_0=L_0$.
\end{proof}

The next intermediate theorem controls the two full-sample excess losses
needed in the proof of
Theorem~\ref{theorem:STLC-delta-ell-f-excess-risk-upper-bound}.
\begin{theorem}\label{theorem:STLC-delta-star-ell-f}
Suppose that Assumption~\ref{assumption:main} holds and that the minima used
to select $\hat f_m$ and $\hat f_u$ are attained for every split $Z$.  Let
$\psi^*_{u,m}$ be a sub-root function with positive fixed point $r^*$ that
satisfies (\ref{eq:STLC-cond-um-delta-star-ell-f-psi-star}).
There is a positive constant $c_{\Delta}$ depending only on $B$ and $L_0$
such that, for every $x>0$, with probability at least
$1-2\exp(-x)$ over $Z$,
\bal\label{eq:STLC-bound-delta-star-ell-f}
\cL_n(\ell_{\hat f_u}-\ell_{f_n^*}) \le c_{\Delta} \pth{r^* + \frac{x}{N_{u,m}}}, \quad \cL_n(\ell_{\hat f_m}-\ell_{f_n^*}) \le c_{\Delta} \pth{r^* + \frac{x}{N_{u,m}}},
\eal
\end{theorem}
\begin{proof}
\noindent\textit{Proof outline.}  Choose the localization coefficient so that
$B/K_1<1$, apply Lemma~\ref{lemma:concentration-gu-gm-Delta*} to the two
empirical minimizers, and absorb the resulting multiple of each full-sample
excess loss.
Set $K_1\defeq2\max\set{1,B}$, so $K_1>1$ and $B/K_1\le1/2$.
Apply Lemma~\ref{lemma:concentration-gu-gm-Delta*} with its parameter
$K_0$ set equal to $K_1$.  Both bounds then hold simultaneously on all of
$\Delta_{\cF}^*$ with probability at least $1-2\exp(-x)$.
Here $\hat c_{\Delta}=\max\set{8K_1d_0^2,\,
8K_1c_0^2+32L_0+c_0}$, with $c_0,d_0$ as in Theorem~\ref{theorem:TLC}.
Choose $c_{\Delta}\defeq\hat c_{\Delta}/(1-B/K_1)$.
Because the event is uniform in $h$, it applies to either empirical
minimizer selected after observing the split.

For $h_u\defeq\ell_{\hat f_u}-\ell_{f_n^*}$, the optimality of $\hat f_u$ gives $\cL_u(h_u)\le0$. Substitution into (\ref{eq:concentration-gu--Delta*-fixedpoint}) therefore gives
$\pth{1-B/K_1}\cL_n(h_u)\le
\hat c_{\Delta}\pth{r^*+x/N_{u,m}}$.
Similarly, for $h_m\defeq\ell_{\hat f_m}-\ell_{f_n^*}$, the optimality of $\hat f_m$ gives $\cL_m(h_m)\le0$, and (\ref{eq:concentration-gm-Delta*-fixedpoint}) yields the same bound with $h_u$ replaced by $h_m$. These are the two inequalities in (\ref{eq:STLC-bound-delta-star-ell-f}). Since $K_1$ depends only on $B$, the constant $c_{\Delta}$ depends only on $B$ and $L_0$.
\end{proof}

\begin{proof}
[\textbf{\textup{Proof of Theorem~\ref{theorem:STLC-delta-ell-f-excess-risk-upper-bound}}}]
\noindent\textit{Proof outline.}  We first intersect the uniform deviation
event with the event controlling both full-sample excess losses.  We then
bound the surrogate for the difference of the two empirical minimizers.
Finally, training optimality removes its training average, and substitution
of the two excess-loss bounds proves the result.

\noindent\textit{Step 1. Work on a common uniform event.}
Lemma~\ref{lemma:STLC-delta-ell-f} verifies the envelope and second-moment
hypotheses for $\Delta_{\cF}$ and the surrogate (\ref{eq:tTn-def}).
The assumed condition on $\psi_{u,m}$ therefore allows us to apply
Theorem~\ref{theorem:TLC}.  Its event has probability at least
$1-\exp(-x)$, and the event in
Theorem~\ref{theorem:STLC-delta-star-ell-f} has probability at least
$1-2\exp(-x)$.
Their intersection has probability at least $1-3\exp(-x)$ by a union
bound, which does not require independence.
On this event, put
$h_{m,u}\defeq\ell_{\hat f_m}-\ell_{\hat f_u}\in\Delta_{\cF}$.
The uniform deviation inequality applies to this random choice.

\noindent\textit{Step 2. Control its surrogate by the two excess losses.}
Using the particular representation
$h_{m,u}=\ell_{\hat f_m}-\ell_{\hat f_u}$ in (\ref{eq:tTn-def}) gives
$\tT_n(h_{m,u})\le2B\bth{
\cL_n(\ell_{\hat f_m}-\ell_{f_n^*})+
\cL_n(\ell_{\hat f_u}-\ell_{f_n^*})}$.
Applying (\ref{eq:STLC-bound-g-upper-bound}) with $h=h_{m,u}$ and $\cH=\Delta_{\cF}$ therefore yields
\bals
\cE(\hat f_m)=\cL_u(h_{m,u})&\le \cL_m(h_{m,u})
+ \frac {2B}{K_0}\cL_n(\ell_{\hat f_m}-\ell_{f_n^*})\\
&\quad+\frac {2B}{K_0}\cL_n(\ell_{\hat f_u}-\ell_{f_n^*})
+c_1 r_{u,m}+\frac{c_2x}{N_{u,m}}.
\eals
\noindent\textit{Step 3. Use empirical optimality and substitute.}
Since $\hat f_m$ minimizes training loss and $\hat f_u\in\cF$,
$\cL_m(h_{m,u})=\cL_m(\ell_{\hat f_m})-\cL_m(\ell_{\hat f_u})\le0$.
Substituting the two bounds in
(\ref{eq:STLC-bound-delta-star-ell-f}) therefore gives
\bals
\cE(\hat f_m)&\le c_1r_{u,m}+\frac{4Bc_{\Delta}}{K_0}\pth{r^*+\frac{x}{N_{u,m}}}+\frac{c_2x}{N_{u,m}}\\
&=c_1r_{u,m}+\frac{4Bc_{\Delta}r^*}{K_0}+\frac{c_3x}{N_{u,m}},
\eals
where $c_3\defeq c_2+4Bc_{\Delta}/K_0$. This proves (\ref{eq:STLC-ell-f-excess-risk-upper-bound}).

\end{proof}

\subsection{Proof of Theorem~\ref{theorem:STLC-delta-ell-f-excess-risk-upper-bound-VC-dim}}

\begin{proof}
[\textbf{Proof of
Theorem~\ref{theorem:STLC-delta-ell-f-excess-risk-upper-bound-VC-dim}}]
\noindent\textit{Proof outline.}  Reduce realizable squared loss to an
indicator class, transfer full-sample localization to with-replacement
empirical $L_2$ localization by a relative VC bound, apply contraction and
entropy integration, solve the resulting sub-root fixed point, and absorb the
full-sample loss after invoking Theorem~\ref{theorem:TLC}. 
We note that the proof strategy of this proof largely follows from~\citet[Theorem 6.1]{yang2026improvedgeneralizationboundstransductive}, and this proof uses 
Theorem~\ref{theorem:TLC} instead of~\citet[Theorem 4.6]{yang2026improvedgeneralizationboundstransductive} so as to avoid the logarithmic confidence factor in~\citet[Theorem 4.6]{yang2026improvedgeneralizationboundstransductive}.

We apply Theorem~\ref{theorem:TLC} to the class
$\cH\defeq\set{\ell_f\colon f\in\cF}$.
By realizability,
$\ell_f(i)=(f(\bx_i)-y_i)^2=(f(\bx_i)-f^*(\bx_i))^2
=\indict{f(\bx_i)\neq y_i}$ for all $i\in[n]$.  Hence $0\in\cH$,
every $h\in\cH$ is indicator-valued, $h^2=h$, and
$T_n(h)=\cL_n(h)$.  We may therefore apply Theorem~\ref{theorem:TLC}
with $H_0=1$ and $\tT_n=T_n$. Moreover, $\cH^2(r)=\cH(r)$.  It follows
from Theorem~\ref{theorem:TC-RC} that if $\psi_{u,m}$ is a sub-root
function satisfying
\bal\label{eq:psi-u-m-goal}
\psi_{u,m}(r) \ge 2\max\set{
\Expect{}{\sup_{h \colon h \in \cH,\cL_n(h) \le r}
R^{(\textup{ind})}_{\bsigma,\bY^{(u)}}h},
\Expect{}{\sup_{h \colon h \in \cH,\cL_n(h) \le r}
R^{(\textup{ind})}_{\bsigma,\bY^{(m)}}h}},
\eal
then $\psi_{u,m}$ meets condition~(\ref{eq:STLC-cond-psi-general}) of
Theorem~\ref{theorem:TLC}.  We now derive upper bounds for
the two expectations on the right-hand side of
(\ref{eq:psi-u-m-goal}).

Let $\cF-f^*$ denote the evaluation-difference class
$\set{g_f\colon f\in\cF}$ on $[n]$, where
$g_f(i)\defeq f(\bx_i)-f^*(\bx_i)$.  Because the VC-dimension of $\cF$
is $\dVC$, the binary-valued class
$\set{g^2\colon g\in\cF-f^*}=\set{\abth{g}\colon g\in\cF-f^*}=\cH$
has VC-dimension at most $\dVC$: on every finite set, its label vectors
are obtained from those of $\cF$ by taking symmetric difference with the
fixed label vector of $f^*$.  Applying Proposition~\ref{proposition:concentration-vc-class-ind}
with $K_0=2$ to the uniform distribution on $[n]$, and then
weakening the resulting coefficient $3/2$ of $T_n(g)$ to $2$, we
obtain that for an absolute constant $c>0$ and every $x'>0$, with
probability at least $1-\exp(-x')$ over $\bY^{(u)}$,
\bal\label{eq:TLC-delta-ell-f-excess-risk-upper-bound-VC-dim-seg1}
\frac1u\sum_{i=1}^u g^2(Y_i)-2T_n(g)
\le c\pth{\frac{\dVC\log(ue/\dVC)}u+\frac{x'}u},
\qquad g\in\cF-f^*.
\eal
Define $\cA\defeq\set{\bY^{(u)}\colon
\textup{(\ref{eq:TLC-delta-ell-f-excess-risk-upper-bound-VC-dim-seg1})
holds}}$.  Then $\Prob{\cA^c}\le\exp(-x')$.  Put
$C(u,\dVC,x')\defeq
c\pth{\dVC\log(ue/\dVC)/u+x'/u}$ and
$b_u(r)\defeq\min\set{2\sqrt{2r+C(u,\dVC,x')},2}$.

We next derive the upper bound for the first expectation in
(\ref{eq:psi-u-m-goal}).  Let $P_u(\bY^{(u)})$ denote the empirical
probability measure of $\bY^{(u)}$, so that
$\norm{g}{L^2(P_u(\bY^{(u)}))}^2=u^{-1}\sum_{i=1}^u g^2(Y_i)$.
Let $C'_1,C'_2>0$ be sufficiently large absolute constants.  We prove below
that they may be chosen so that
\bal\label{eq:TLC-delta-ell-f-excess-risk-upper-bound-VC-dim-u}
&\Expect{}{\sup_{h\colon h\in\cH,\ \cL_n(h)\le r}
R^{(\textup{ind})}_{\bsigma,\bY^{(u)}}h}
=\Expect{\bY^{(u)},\bsigma}{\sup_{g\in\cF-f^*\colon T_n(g)\le r}
R^{(\textup{ind})}_{\bsigma,\bY^{(u)}}g^2}\nonumber\\
&\stackrel{\circled{1}}{\le}
2\Expect{\bY^{(u)},\bsigma}{\sup_{g\in\cF-f^*\colon T_n(g)\le r}
R^{(\textup{ind})}_{\bsigma,\bY^{(u)}}g}\nonumber\\
&\stackrel{\circled{2}}{=}
2\Expect{\bY^{(u)},\bsigma}{\indict{\bY^{(u)}\in\cA}
\sup_{g\in\cF-f^*\colon T_n(g)\le r}
R^{(\textup{ind})}_{\bsigma,\bY^{(u)}}g}\nonumber\\
&\phantom{=}+2\Expect{\bY^{(u)},\bsigma}{\indict{\bY^{(u)}\in\cA^c}
\sup_{g\in\cF-f^*\colon T_n(g)\le r}
R^{(\textup{ind})}_{\bsigma,\bY^{(u)}}g}\nonumber\\
&\stackrel{\circled{3}}{\le}
2\Expect{\bY^{(u)},\bsigma}{\indict{\bY^{(u)}\in\cA}
\sup_{\substack{g\in\cF-f^*\colon\\
\norm{g}{L^2(P_u(\bY^{(u)}))}^2\le2r+C(u,\dVC,x')}}
R^{(\textup{ind})}_{\bsigma,\bY^{(u)}}g}\nonumber\\
&\phantom{=}+2\Expect{\bY^{(u)},\bsigma}{\indict{\bY^{(u)}\in\cA^c}
\sup_{g\in\cF-f^*\colon T_n(g)\le r}
R^{(\textup{ind})}_{\bsigma,\bY^{(u)}}g}\nonumber\\
&\stackrel{\circled{4}}{\le}\frac{2C_0}{\sqrt u}
\Expect{\bY^{(u)}}{\indict{\bY^{(u)}\in\cA}
\int_0^{b_u(r)}\sqrt{\log N\pth{\cF-f^*,
\norm{\cdot}{L^2(P_u(\bY^{(u)}))},\eps}}\diff\eps}\nonumber\\
&\phantom{=}+\frac{2C_0\exp(-x')}{\sqrt u}
\sup_{\by\in[n]^u}\int_0^2\sqrt{\log N\pth{\cF-f^*,
\norm{\cdot}{L^2(P_u(\by))},\eps}}\diff\eps\nonumber\\
&\stackrel{\circled{5}}{\le}
C'_1\sqrt{\frac{\dVC\log(ue/\dVC)}u}\sqrt r
+C'_2\frac{\dVC\log(ue/\dVC)}u.
\eal
Here $C_0$ is the absolute constant in
Theorem~\ref{theorem:dudley-integral-entropy-bound}.
$\circled{1}$ follows from the contraction property in
Theorem~\ref{theorem:contraction-RC}, because $z\mapsto z^2$ is
$2$-Lipschitz on $[-1,1]$ and vanishes at zero.  Step $\circled{2}$ is
the decomposition over $\cA$ and $\cA^c$.  When
$\bY^{(u)}\in\cA$ and $T_n(g)\le r$, we have
$u^{-1}\sum_{i=1}^u g^2(Y_i)
\le2T_n(g)+C(u,\dVC,x')\le2r+C(u,\dVC,x')$, which proves
$\circled{3}$.

We next justify each term in $\circled{4}$.  Put
$s_u(r)\defeq2r+C(u,\dVC,x')$.  For a deterministic sequence
$\by=(y_1,\ldots,y_u)\in[n]^u$, write
$P_u(\by)\defeq u^{-1}\sum_{i=1}^u\delta_{y_i}$, and define
$\cG_{\by}(r)\defeq\set{g\in\cF-f^*\colon
\norm{g}{L^2(P_u(\by))}^2\le s_u(r)}$ and
$\cG_n(r)\defeq\set{g\in\cF-f^*\colon T_n(g)\le r}$.
Both classes contain the zero function because $f^*\in\cF$.  Also, every
$g\in\cF-f^*$ takes values in $\set{-1,0,1}$.

First consider the good-event term.  For $g_1,g_2\in\cG_{\by}(r)$, the
triangle inequality and the pointwise bound $\abth{g_1-g_2}\le2$ imply
$
\norm{g_1-g_2}{L^2(P_u(\by))}
\le\min\set{2\sqrt{s_u(r)},2}=b_u(r).
$ Hence the empirical $L^2(P_u(\by))$-diameter of $\cG_{\by}(r)$ is at most
$b_u(r)$.
Conditional on $\bY^{(u)}=\by$, Dudley's inequality and the
external-covering convention in
Theorem~\ref{theorem:dudley-integral-entropy-bound} yield
\bal\label{eq:VC-circled-four-good-term}
&2\Expect{\bsigma}{\indict{\by\in\cA}
\sup_{g\in\cG_{\by}(r)}R^{(\textup{ind})}_{\bsigma,\by}g}\nonumber\\*
&\quad\le\frac{2C_0}{\sqrt u}\indict{\by\in\cA}
\int_0^{b_u(r)}\sqrt{\log N\pth{\cF-f^*,
\norm{\cdot}{L^2(P_u(\by))},\eps}}\diff\eps.
\eal
Here the theorem first gives an upper integration limit no larger than
$b_u(r)/2$.  We enlarged it to $b_u(r)$ because the integrand is
nonnegative.  The localized class is a subset of $\cF-f^*$, so its external
covering number at every radius is bounded by that of $\cF-f^*$.  Taking
expectation in $\by=\bY^{(u)}$ in
(\ref{eq:VC-circled-four-good-term}) gives the first term on the right-hand
side of $\circled{4}$.

For the bad-event term, the pointwise range of $\cF-f^*$ gives
$\sup_{g_1,g_2\in\cG_n(r)}
\norm{g_1-g_2}{L^2(P_u(\by))}\le2$ for every $\by\in[n]^u$.
Another conditional application of Dudley's inequality, followed by
enlarging the upper integration limit from at most $1$ to $2$, gives
\bal\label{eq:VC-circled-four-bad-term}
&2\Expect{\bY^{(u)},\bsigma}{\indict{\bY^{(u)}\in\cA^c}
\sup_{g\in\cG_n(r)}R^{(\textup{ind})}_{\bsigma,\bY^{(u)}}g}\nonumber\\
&\quad\le\frac{2C_0}{\sqrt u}
\Expect{\bY^{(u)}}{\indict{\bY^{(u)}\in\cA^c}
\int_0^2\sqrt{\log N\pth{\cF-f^*,
\norm{\cdot}{L^2(P_u(\bY^{(u)}))},\eps}}\diff\eps}\nonumber\\
&\quad\le\frac{2C_0\exp(-x')}{\sqrt u}
\sup_{\by\in[n]^u}\int_0^2\sqrt{\log N\pth{\cF-f^*,
\norm{\cdot}{L^2(P_u(\by))},\eps}}\diff\eps.
\eal
The last inequality uses $\Prob{\cA^c}\le\exp(-x')$ and bounds the
nonnegative entropy integral by its supremum over $[n]^u$.  This is exactly
the second term in $\circled{4}$ by
(\ref{eq:VC-circled-four-bad-term}).  The VC entropy estimate below guarantees
that all covering numbers and entropy integrals used here are finite.

We now justify each term in $\circled{5}$.  For a probability measure $Q$ on
$[n]$ and $0\le b\le2$, define the entropy integral
$\mathsf E_Q(b)\defeq\int_0^b\sqrt{\log N\pth{\cF-f^*,
\norm{\cdot}{L^2(Q)},\eps}}\diff\eps$.
Denote the first and second terms on the right-hand side of $\circled{4}$ by
\bal\label{eq:VC-circled-five-terms}
\mathsf G_u(r) \defeq\frac{2C_0}{\sqrt u}
\Expect{\bY^{(u)}}{\indict{\bY^{(u)}\in\cA}
\mathsf E_{P_u(\bY^{(u)})}(b_u(r))}, \quad
\mathsf B_u \defeq\frac{2C_0\exp(-x')}{\sqrt u}
\sup_{\by\in[n]^u}\mathsf E_{P_u(\by)}(2).
\eal
The translation $f\mapsto f-f^*$ preserves every $L^2(Q)$ distance.  The
uniform VC entropy bound~\citep[Theorem 2.6.7]{van1996weak} therefore supplies
an absolute constant $C\ge2e$ such that, for every probability measure $Q$
and $0<\eps\le2$,
$\log N\pth{\cF-f^*,\norm{\cdot}{L^2(Q)},\eps}
\le C\dVC\log(C/\eps)$.
The cited bound is usually stated for internal covering numbers.  It remains
valid under our external-covering convention because allowing centers in the
ambient space can only decrease a covering number.  Consequently, for
$0<b\le2$,
\bal\label{eq:TLC-delta-ell-f-excess-risk-upper-bound-VC-dim-seg2}
\mathsf E_Q(b)
&\le C\sqrt{\dVC}\int_0^b\sqrt{\log\pth{\frac C\eps}}\diff\eps\nonumber\\
&\stackrel{\circled{6}}{\le}
C\sqrt{\dVC b}\sqrt{\int_0^b\log\pth{\frac C\eps}\diff\eps}
\stackrel{\circled{7}}{=}
Cb\sqrt{\dVC\pth{\log\pth{\frac Cb}+1}}
\le Cb\sqrt{\dVC\log\pth{\frac Cb}}.
\eal
Here $\circled{6}$ is the Cauchy--Schwarz inequality, and
$\circled{7}$ follows from
$\int_0^b\log(C/\eps)\diff\eps=b\log(C/b)+b$.  The final inequality uses
$C\ge2e$ and $b\le2$, which imply $\log(C/b)\ge1$, and enlarges the absolute
constant $C$.  In particular,
\bal\label{eq:TLC-delta-ell-f-excess-risk-upper-bound-VC-dim-seg3}
\sup_Q\mathsf E_Q(2)\le C\sqrt{\dVC},
\eal
where the supremum ranges over all probability measures on $[n]$.

Set $L_u\defeq\log(ue/\dVC)$, $x'\defeq\dVC L_u$, and
$a_u\defeq\dVC L_u/u$.
Writing $q=u/\dVC\ge1$, we have $L_u=1+\log q\le q$, and hence
$\dVC/u\le a_u\le1$.  Moreover, $C(u,\dVC,x')\le Ca_u$ and
$\exp(-x')=\pth{\dVC/(ue)}^{\dVC}\le\dVC/u\le a_u$.

We first bound the bad-event term.  Equations
(\ref{eq:VC-circled-five-terms}) and
(\ref{eq:TLC-delta-ell-f-excess-risk-upper-bound-VC-dim-seg3}) give
\bal\label{eq:VC-circled-five-bad-bound}
\mathsf B_u
\le C\exp(-x')\sqrt{\frac{\dVC}{u}}
\le Ca_u.
\eal
The last inequality also uses $\sqrt{\dVC/u}\le1$.

We next bound the good-event term directly from the upper bound on $b_u(r)$.
Put $t_u(r)\defeq r+a_u$, $s_u(r)\defeq\min\set{\sqrt{t_u(r)},1}$, and
$\bar b_u(r)\defeq\min\set{C_b\sqrt{t_u(r)},2}$ for a sufficiently large
absolute constant $C_b\ge2$.  Since $C(u,\dVC,x')\le Ca_u$, the definition of
$b_u(r)$ gives $b_u(r)\le\bar b_u(r)$ for every $r\ge0$, while
$2s_u(r)\le\bar b_u(r)\le C_bs_u(r)$.

Enlarge the entropy constant $C$ in
(\ref{eq:TLC-delta-ell-f-excess-risk-upper-bound-VC-dim-seg2}), if necessary,
so that $C\ge2e$, and put $H(b)\defeq b\sqrt{\log(C/b)}$ for $0<b\le2$.
Since $H'(b)=\pth{2\log(C/b)-1}/\pth{2\sqrt{\log(C/b)}}\ge0$, the function $H$
is nondecreasing on $(0,2]$.  Thus
(\ref{eq:TLC-delta-ell-f-excess-risk-upper-bound-VC-dim-seg2}) and
$\indict{\bY^{(u)}\in\cA}\le1$ give
$\mathsf G_u(r)\le C\sqrt{\dVC/u}\,H\pth{b_u(r)}
\le C\sqrt{\dVC/u}\,H\pth{\bar b_u(r)}$.
\par\vfil\pagebreak

Moreover, $t_u(r)\ge a_u$, $a_u\le1$, and hence
$s_u(r)\ge\sqrt{a_u}$.  The identities $a_u=L_u/q$ and
$L_u=1+\log q$ therefore yield
$\log\pth{C/(2s_u(r))}\le\log\pth{C/(2\sqrt{a_u})}
\le CL_u$.  Using $2s_u(r)\le\bar b_u(r)\le C_bs_u(r)$ and
$s_u(r)\le\sqrt{t_u(r)}$, we have
$H\pth{\bar b_u(r)}\le
C_bs_u(r)\sqrt{\log\pth{C/(2s_u(r))}}
\le Cs_u(r)\sqrt{L_u}$.  It follows that
\bal\label{eq:VC-circled-five-good-bound}
\mathsf G_u(r)
\le C\sqrt{\frac{\dVC}{u}}s_u(r)\sqrt{L_u}
\le C\sqrt{a_u(r+a_u)}
\le C\pth{\sqrt{a_ur}+a_u},\qquad r\ge0.
\eal
Together, (\ref{eq:VC-circled-five-bad-bound}) and
(\ref{eq:VC-circled-five-good-bound}) prove $\circled{5}$ after choosing
$C'_1$ and $C'_2$ as sufficiently large absolute constants.

Repeating the argument for
(\ref{eq:TLC-delta-ell-f-excess-risk-upper-bound-VC-dim-u}) with $u$
replaced by $m$ gives
\bal\label{eq:TLC-delta-ell-f-excess-risk-upper-bound-VC-dim-m}
\Expect{}{\sup_{h\in\cH\colon\cL_n(h)\le r}
R^{(\textup{ind})}_{\bsigma,\bY^{(m)}}h}
\le C'_1\sqrt{\frac{\dVC\log(me/\dVC)}m}\sqrt r
+C'_2\frac{\dVC\log(me/\dVC)}m.
\eal
Because $p\mapsto\log(pe/\dVC)/p$ is nonincreasing for
$p\ge\dVC$ and $u\ge m$, equations
(\ref{eq:TLC-delta-ell-f-excess-risk-upper-bound-VC-dim-u}),
(\ref{eq:TLC-delta-ell-f-excess-risk-upper-bound-VC-dim-m}), and
(\ref{eq:psi-u-m-goal}) show that we may take
$\psi_{u,m}(r)=2C'_1\sqrt{\dVC\log(me/\dVC)/m}\sqrt r
+2C'_2\dVC\log(me/\dVC)/m$,
after enlarging the absolute constants if necessary.  This function is
sub-root.  For every $0\le r\le r_{u,m}$, we have
$r\le\psi_{u,m}(r)$. Solving this quadratic inequality in $\sqrt r$
gives
\bal\label{eq:TLC-delta-ell-f-excess-risk-upper-bound-VC-dim-seg3-post}
r_{u,m}\le C\frac{\dVC\log(me/\dVC)}m.
\eal
Theorem~\ref{theorem:TLC}, applied with $K_0=2$, $H_0=1$, and
$N_{u,m}=m$, together with
(\ref{eq:TLC-delta-ell-f-excess-risk-upper-bound-VC-dim-seg3-post}),
now implies that, for every $x>0$, with probability at least
$1-\exp(-x)$, every $h\in\cH$ satisfies
\bal\label{eq:TLC-delta-ell-f-excess-risk-upper-bound-VC-dim-seg4}
\cL_u(h)\le\cL_m(h)+\frac12\cL_n(h)
+C\frac{\dVC\log(me/\dVC)}m+C\frac xm.
\eal
Set $h=\ell_{\hat f_m}$.  Realizability and empirical risk minimization
give $0\le\cL_m(\ell_{\hat f_m})\le\cL_m(\ell_{f^*})=0$, and hence
$\cL_m(h)=0$.  Moreover,
$\cL_n(h)=(u/n)\cL_u(h)\le\cL_u(h)$.  Substituting these identities into
(\ref{eq:TLC-delta-ell-f-excess-risk-upper-bound-VC-dim-seg4}) and
absorbing $\cL_u(h)/2$ into the left-hand side proves
(\ref{eq:STLC-delta-ell-f-excess-risk-upper-bound-VC-dim}).
\end{proof}

\subsection{Proof of Theorem~\ref{theorem:STLC-kernel}: STLC-Based Excess Risk Bound for Transductive Kernel Learning}
\label{sec:TKL-proof}

Put $\bK_n\defeq\bK/n$ and $q_0\defeq\rank(\bK_n)$.  Thus exactly
$q_0$ of the ordered eigenvalues
$\hat\lambda_1\ge\cdots\ge\hat\lambda_n\ge0$ are positive.

Define the empirical covariance operator $\hat T_n\colon\cH_K\to\cH_K$ by
$\hat T_ng\defeq n^{-1}\sum_{i=1}^nK(\cdot,\bx_i)g(\bx_i)$ for
$g\in\cH_K$.
The nonzero eigenvalues of $\hat T_n$ are
$\set{\hat \lambda_q}_{q=1}^{q_0}$.  Choose an orthonormal family of
eigenvectors $\bU^{(1)},\ldots,\bU^{(q_0)}$ of $\bK_n$, where
$\bU^{(q)}$ corresponds to $\hat\lambda_q$, and define
$\Phi_q\defeq(n\hat\lambda_q)^{-1/2}
\sum_{j=1}^nK(\cdot,\bx_j)\bth{\bU^{(q)}}_j$.
Then $\set{\Phi_q}_{q=1}^{q_0}$ is an orthonormal basis of $\cH_{\bX_n}$ and $\hat T_n\Phi_q=\hat\lambda_q\Phi_q$ for every $q\in[q_0]$.

The following spectral estimate is the main input to the kernel proof.
\begin{lemma}\label{lemma:STLC-kernel-ind}
Let $\cF=\cH_{\bX_n}(\mu)$.  For $r\ge0$, define
\bal
{\tilde\varphi_u}(r)&\defeq
\min_{Q\in\set{0,\ldots,n}}\pth{\sqrt{\frac{rQ}{u}}+\mu
\sqrt{\frac{\sum_{q=Q+1}^{n}\hat\lambda_q}{u}}},
\label{eq:varphi-STLC-kernel-u-ind}\\
{\tilde\varphi_m}(r)&\defeq
\min_{Q\in\set{0,\ldots,n}}\pth{\sqrt{\frac{rQ}{m}}+\mu
\sqrt{\frac{\sum_{q=Q+1}^{n}\hat\lambda_q}{m}}}.
\label{eq:varphi-STLC-kernel-m-ind}
\eal
Then, for every $r\ge0$,
\bal
\Expect{\bY^{(u)},\bsigma}{\sup_{f\in \cF \colon T_n(f)\le r}
R^{(\textup{ind})}_{\bsigma,\bY^{(u)}}f}&\le\tilde\varphi_u(r),
\label{eq:STLC-kernel-u-ind}\\
\Expect{\bY^{(m)},\bsigma}{\sup_{f\in \cF \colon T_n(f)\le r}
R^{(\textup{ind})}_{\bsigma,\bY^{(m)}}f}&\le\tilde\varphi_m(r).
\label{eq:STLC-kernel-m-ind}
\eal
\end{lemma}
\begin{proof}
\noindent\textit{Proof outline.}  Expand each function in the empirical
covariance eigenbasis, split the Rademacher sum into its first $Q$ and remaining
coordinates, and control the two parts by the empirical $L_2$ radius and RKHS
radius, respectively.
If $q_0=0$, then $\cH_{\bX_n}=\set{0}$ and all the empirical eigenvalues
vanish, so both conclusions are immediate.  We henceforth assume $q_0\ge1$.

We have
\bal\label{eq:lemma-STLC-empirical-NN-seg1}
R^{(\textup{ind})}_{\bsigma,\bY^{(u)}}f = \frac 1u \sum\limits_{i=1}^u {\sigma_i}{f(\bx_{Y_i})} =
\iprod{f}
{\frac 1u \sum\limits_{i=1}^u {\sigma_i}{K(\cdot,\bx_{Y_i})}}_{\cH_K}.
\eal
Because $\set{\Phi_q}_{q=1}^{q_0}$ is an orthonormal basis of
$\cH_{\bX_n}$, for every $Q\in\set{0,\ldots,q_0}$ we further express the
right-hand side of (\ref{eq:lemma-STLC-empirical-NN-seg1}) as
\bal\label{eq:lemma-STLC-empirical-NN-seg2}
\iprod{f}{\frac 1u \sum\limits_{i=1}^u {\sigma_i}{K(\cdot,\bx_{Y_i})}}_{\cH_K}
&=\iprod{\sum\limits_{q=1}^{Q} \sqrt{\hat \lambda_q}  \iprod{f}
{\Phi_q}_{\cH_K}\Phi_q }
{v^{(Q)}(\bY^{(u)},\bsigma)}_{\cH_K} \nonumber \\
&+\iprod{\bar f}
{{\bar v}^{(Q)}(\bY^{(u)},\bsigma)}_{\cH_K},
\eal
where
\bals
\bar f &\defeq f - \sum\limits_{q=1}^{Q}   \iprod{f}
{\Phi_q}_{\cH_K}\Phi_q, \\
v^{(Q)}(\bY^{(u)},\bsigma) &\defeq \frac 1u\sum\limits_{q=1}^{Q} \frac{1}{\sqrt{\hat \lambda_q}}\iprod{\sum\limits_{i=1}^u {\sigma_i}{K(\cdot,\bx_{Y_i})}}
{\Phi_q}_{\cH_K}\Phi_q, \\
{\bar v}^{(Q)}(\bY^{(u)},\bsigma) &\defeq
\frac 1u \sum\limits_{q = Q+1}^{q_0} \iprod{\sum\limits_{i=1}^u {\sigma_i}{K(\cdot,\bx_{Y_i})}}{\Phi_q}_{\cH_K}\Phi_q.
\eals
For the empirical covariance operator $\hat T_n$ defined above, the
reproducing property gives the quadratic-form identity
$\iprod{\hat T_n f}{f}_{\cH_K}
=\iprod{n^{-1}\sum_{i=1}^nK(\cdot,\bx_i)f(\bx_i)}{f}_{\cH_K}
=T_n(f)$.
As a result,
\bal\label{eq:lemma-STLC-empirical-NN-seg3}
\norm{\sum\limits_{q=1}^Q \sqrt{\hat \lambda_q}  \iprod{f}
{\Phi_q}_{\cH_K}\Phi_q }{\cH_{K}}^2
&=\sum\limits_{q=1}^Q \hat \lambda_q\iprod{f}{\Phi_q}_{\cH_K}^2\nonumber\\
&\le\sum\limits_{q=1}^{q_0}\hat \lambda_q\iprod{f}{\Phi_q}_{\cH_K}^2
=\iprod{\hat T_n f}{f}_{\cH_K}\nonumber\\
&=T_n(f)\le r,
\eal
which holds for all $f$ such that $T_n(f)\le r$.

Combining (\ref{eq:lemma-STLC-empirical-NN-seg1}), (\ref{eq:lemma-STLC-empirical-NN-seg2}), and (\ref{eq:lemma-STLC-empirical-NN-seg3}), we have
\bal\label{eq:lemma-STLC-empirical-NN-seg4}
&\Expect{\bY^{(u)},\bsigma}{\sup_{f\in \cF \colon T_n(f)\le r} R^{(\textup{ind})}_{\bsigma,\bY^{(u)}}f}\nonumber \\
&\stackrel{\circled{1}}{\le} \sup_{f\in \cF \colon T_n(f)\le r}
\norm{\sum\limits_{q=1}^Q \sqrt{\hat \lambda_q}  \iprod{f}
{\Phi_q}_{\cH_K}\Phi_q}{\cH_K}
\cdot
\Expect{\bY^{(u)},\bsigma}{\norm{v^{(Q)}(\bY^{(u)},\bsigma)}{\cH_K}} \nonumber \\
&\phantom{=}+\sup_{f\in \cF \colon T_n(f)\le r}\norm{\bar f}{\cH_K} \cdot
\Expect{\bY^{(u)},\bsigma}{\norm{{\bar v}^{(Q)}(\bY^{(u)},\bsigma)}{\cH_K}}
\nonumber \\
&\le \sqrt{r} \Expect{\bY^{(u)},\bsigma}{\norm{v^{(Q)}(\bY^{(u)},\bsigma)}{\cH_K}}
+ \mu \Expect{\bY^{(u)},\bsigma}{\norm{{\bar v}^{(Q)}(\bY^{(u)},\bsigma)}{\cH_K}}.
\eal
Here $\circled{1}$ follows from the Cauchy--Schwarz inequality. The last
line also uses that orthogonal projection is contractive, so
$\norm{\bar f}{\cH_K}\le\norm{f}{\cH_K}\le\mu$.

We have
\bal\label{eq:lemma-STLC-empirical-NN-seg5}
&\frac 1u \Expect{\bY^{(u)},\bsigma}{\iprod{
\sum\limits_{i=1}^u {\sigma_i}{K(\cdot,\bx_{Y_i})}}{\Phi_q}_{\cH_K}^2} \stackrel{\circled{1}}{=}\frac 1u \Expect{\bY^{(u)}}{\sum\limits_{i=1}^u
\iprod{K(\cdot,\bx_{Y_i})}{\Phi_q}_{\cH_K}^2}  \nonumber \\
&= \frac 1u \Expect{\bY^{(u)}}{\sum\limits_{i=1}^u
\Phi_q(\bx_{Y_i})^2}
=\frac 1n \sum\limits_{i=1}^n \Phi^2_q(\bx_i)  = \iprod{\hat T_n\Phi_q}{\Phi_q}_{\cH_K} = \hat \lambda_q.
\eal
Here $\circled{1}$ follows because the Rademacher variables are independent and
centered and have unit second moments, so all cross terms vanish.
It follows from (\ref{eq:lemma-STLC-empirical-NN-seg5}) that
\bal\label{eq:lemma-STLC-empirical-NN-seg6}
&\Expect{\bY^{(u)},\bsigma}{\norm{v^{(Q)}(\bY^{(u)},\bsigma)}{\cH_K}}
= \frac 1{\sqrt u} \Expect{\bY^{(u)},\bsigma}{\sqrt{\frac 1u \sum\limits_{q=1}^Q \frac{1}{\hat \lambda_q} \iprod{\sum\limits_{i=1}^u {\sigma_i}{K(\cdot,\bx_{Y_i})}}{\Phi_q}_{\cH_K}^2} } \nonumber \\
&\stackrel{\circled{1}}{\le} \frac 1 {\sqrt u} \sqrt{ \frac 1u
\Expect{\bY^{(u)},\bsigma}{\sum\limits_{q=1}^Q \frac{1}{\hat \lambda_q}\iprod{
\sum\limits_{i=1}^u {\sigma_i}{K(\cdot,\bx_{Y_i})}}{\Phi_q}_{\cH_K}^2}} \stackrel{\circled{2}}{=}\sqrt{\frac Qu},
\eal
where $\circled{1}$ is due to  Jensen's inequality, and $\circled{2}$ follows from (\ref{eq:lemma-STLC-empirical-NN-seg5}). Similarly, we have
\bal\label{eq:lemma-STLC-empirical-NN-seg7}
&\Expect{\bY^{(u)},\bsigma}{\norm{{\bar v}^{(Q)}(\bY^{(u)},\bsigma)}{\cH_K}} = \frac 1{\sqrt u}\Expect{\bY^{(u)},\bsigma}{\sqrt{\frac 1u \sum\limits_{q = Q+1}^{q_0} \iprod{
\sum\limits_{i=1}^u {\sigma_i}{K(\cdot,\bx_{Y_i})}}{\Phi_q}_{\cH_K}^2}} \nonumber \\
&\le \frac 1{\sqrt u} \sqrt{ \frac 1u
\Expect{\bY^{(u)},\bsigma}{\sum\limits_{q = Q+1}^{q_0} \iprod{\sum\limits_{i=1}^u
{\sigma_i}{K(\cdot,\bx_{Y_i})}}{\Phi_q}_{\cH_K}^2 }}
=\sqrt{\frac{\sum\limits_{q = Q+1}^{n}\hat \lambda_q}{u}}.
\eal
It follows from (\ref{eq:lemma-STLC-empirical-NN-seg4}), (\ref{eq:lemma-STLC-empirical-NN-seg6}),
 and (\ref{eq:lemma-STLC-empirical-NN-seg7}) that
\bals
\Expect{\bY^{(u)},\bsigma}{\sup_{f\in \cF \colon T_n(f)\le r}\iprod{f}
{\frac 1u \sum\limits_{i=1}^u {\sigma_i}{K(\cdot,\bx_{Y_i})}}_{\cH_K}}
\le \min_{Q\in\set{0,\ldots,q_0}} \pth{\sqrt{\frac {rQ}u} +
\mu
\sqrt{\frac{\sum\limits_{q = Q+1}^{n}\hat \lambda_q}{u}}}.
\eals
The objective in the last display is nondecreasing for $Q\ge q_0$ because $\hat\lambda_q=0$ for $q>q_0$. Hence its minimum over $Q\in\set{0,\ldots,q_0}$ equals its minimum over $Q\in\set{0,\ldots,n}$, which proves (\ref{eq:STLC-kernel-u-ind}). Replacing $u$ by $m$ proves (\ref{eq:STLC-kernel-m-ind}).
\end{proof}

We can take $\psi_{u,m}$ in Theorem~\ref{theorem:TLC} as the upper bound for inductive Rademacher complexities using Theorem~\ref{theorem:TC-RC}, leading to the following corollary.
\begin{corollary}\label{corollary:STLC-delta-ell-f-ind-rc}
Under the conditions of Theorem~\ref{theorem:TLC} with $\cH=\Delta_{\cF}$
and the surrogate functional $\tT_n$ defined in (\ref{eq:tTn-def}), let
$\psi_{u,m}$ be a sub-root function with positive fixed point $r_{u,m}$.
Put $C_{L_0}\defeq2\max\set{1,2L_0}$.  If, for every $r\ge r_{u,m}$,
\bal\label{eq:STLC-cond-u-delta-ell-f-ind-rc}
\psi_{u,m}(r) \ge C_{L_0}\max\set{\Expect{}{\sup_{h \colon h \in \Delta_{\cF},\tT_n(h) \le r} R^{(\textup{ind})}_{\bsigma,\bY^{(u)}} h}, \Expect{}{\sup_{h \colon h \in \Delta_{\cF},\tT_n(h) \le r} R^{(\textup{ind})}_{\bsigma,\bY^{(m)}} h}},
\eal
Then, for every $x>0$, with probability at least $1-\exp(-x)$ over $Z$,
$\cL_h(Z)\le\cL_h(\barZ)+\tT_n(h)/K_0+c_1r_{u,m}
+c_2x/N_{u,m}$ for every $h\in\Delta_{\cF}$.
\end{corollary}
\begin{proof}
\noindent\textit{Proof outline.}  Dominate transductive complexities by
with-replacement Rademacher complexities and use contraction for the square
class, then invoke Theorem~\ref{theorem:TLC}.
Put $\cA_r\defeq\set{h\in\Delta_{\cF}\colon\tT_n(h)\le r}\cup\set{0}$.
Define $\cA_r^2\defeq\set{h^2\colon h\in\cA_r}$.
Theorem~\ref{theorem:TC-RC} gives
$\cfrakR_p^\eta(\cA_r)\le2\cfrakR_p^{(\textup{ind})}(\cA_r)$.
Because $\abth{h(i)}\le L_0$ and $0\in\cA_r$, the contraction inequality,
applied to $a\mapsto a^2$, gives
$\cfrakR_p^\eta(\cA_r^2)
\le2\cfrakR_p^{(\textup{ind})}(\cA_r^2)
\le4L_0\cfrakR_p^{(\textup{ind})}(\cA_r)$ for
$p\in\set{u,m}$ and $\eta\in\set{+,-}$.
Thus (\ref{eq:STLC-cond-u-delta-ell-f-ind-rc}) implies
(\ref{eq:STLC-cond-psi-general}), and Theorem~\ref{theorem:TLC} proves the
claim.
\end{proof}

\begin{proof}
[\textbf{\textup{Proof of Theorem~\ref{theorem:STLC-kernel}}}]
\noindent\textit{Proof outline.}  Verify existence of the empirical
minimizers, build a sub-root majorant for localized loss differences by
contraction and Lemma~\ref{lemma:STLC-kernel-ind}, construct the analogous
excess-loss majorant, compare their fixed points, and substitute both bounds
into the generic STLC theorem.

The class $\cH_{\bX_n}(\mu)$ is a closed bounded subset of the finite-dimensional Hilbert space $\cH_{\bX_n}$ and hence is compact. Since the loss is Lipschitz and point evaluations are continuous in an RKHS, the training and test empirical risks are continuous on this class. Thus the minimizers $\hat f_m$ and $\hat f_u$ exist for every split.
If $L=0$ or $\cH_{\bX_n}(\mu)=\set{0}$, all predictors in the class have the same pointwise loss or the class is a singleton, respectively. In either case $\cE(\hat f_m)=0$ and (\ref{eq:STLC-kernel-excess-loss}) is immediate. Hence, in the remainder of the proof, assume $L>0$ and $\cH_{\bX_n}(\mu)\neq\set{0}$. Then $\mu>0$, $q_0\ge1$, and the sub-root functions constructed below are nontrivial, so their positive fixed points exist.
Recall that $C_{L_0}=2\max\set{1,2L_0}$, set
$B_K\defeq\max\set{1,L^2B'}$ and $\cF=\cH_{\bX_n}(\mu)$, and let
$\tT_n$ be the functional in (\ref{eq:tTn-def}).  As shown after
Assumption~\ref{assumption:Lipschitz-loss-Tn-f-Ln-ellf},
Assumption~\ref{assumption:main}(2) holds with $B=B_K$.  We use this value
both in $\tT_n$ and when applying the generic excess-risk theorem.  We first
construct its two required sub-root majorants.
Let $h\in\Delta_{\cF}$ satisfy $\tT_n(h)\le r$, where $r>0$.  By the
definition of the infimum in (\ref{eq:tTn-def}), there exist
$f_1,f_2\in\cF$ such that $h=\ell_{f_1}-\ell_{f_2}$ and the objective in
(\ref{eq:tTn-def}) is smaller than $\tT_n(h)+r\le2r$.  Equivalently,
\bal\label{eq:STLC-kernel-surrogate-localization}
\cL_n(\ell_{f_1}-\ell_{f_n^*})
+\cL_n(\ell_{f_2}-\ell_{f_n^*})<\frac r{B_K}.
\eal
Set $B_1\defeq B'/B_K$.
For $r>0$ we have
\bals
&2\Expect{}{\sup_{h \colon h \in \Delta_{\cF},\tT_n(h) \le r} R^{(\textup{ind})}_{\bsigma,\bY^{(u)}}h } \\
&\stackrel{\circled{1}}{\le} 2\Expect{\bY^{(u)},\bsigma}{\sup_{\substack{f_1,f_2 \in \cF \\ \cL_n(\ell_{f_1} - \ell_{f_n^*})+\cL_n(\ell_{f_2} - \ell_{f_n^*})\le r/B_K}} R^{(\textup{ind})}_{\bsigma,\bY^{(u)}}(\ell_{f_1}-\ell_{f_2})} \\
&\le 2\Expect{\bY^{(u)},\bsigma}{\sup_{\substack{f_1\in \cF \\ \cL_n(\ell_{f_1} - \ell_{f_n^*})\le r/B_K}} R^{(\textup{ind})}_{\bsigma,\bY^{(u)}} \pth{\ell_{f_1}- \ell_{f_n^*}}} \\
&\quad\quad+ 2\Expect{\bY^{(u)},\bsigma }{\sup_{\substack{f_2 \in \cF \\ \cL_n(\ell_{f_2} - \ell_{f_n^*})\le r/B_K}} R^{(\textup{ind})}_{\bsigma,\bY^{(u)}}\pth{\ell_{f_n^*}-\ell_{f_2}}} \\
&\stackrel{\circled{2}}{\le} 4L\Expect{\bY^{(u)},\bsigma}{\sup_{f\in \cF \colon
T_n\pth{f - f_n^*} \le B_1r} R^{(\textup{ind})}_{\bsigma,\bY^{(u)}} \pth{f- f_n^*}} \\
&\stackrel{\circled{3}}{\le} 8L\Expect{\bY^{(u)},\bsigma}{\sup_{f\in \cF \colon T_n(f)\le B_1r/4} R^{(\textup{ind})}_{\bsigma,\bY^{(u)}}f}.
\eals
Here $\circled{1}$ follows from
(\ref{eq:STLC-kernel-surrogate-localization}).  Both summands in its left-hand
side are nonnegative by the optimality of $f_n^*$.  Hence their sum being at
most $r/B_K$ implies each individual constraint used in the next line.  For
each $i\in[n]$, let
$\cA_i\defeq\set{f(\bx_i)-f_n^*(\bx_i)\colon f\in\cF}$.
Assumption~\ref{assumption:Lipschitz-loss-Tn-f-Ln-ellf}(1) shows that the
map
$a\mapsto\ell\pth{a+f_n^*(\bx_i),y_i}
-\ell\pth{f_n^*(\bx_i),y_i}$ is $L$-Lipschitz on $\cA_i$ and vanishes at
$0\in\cA_i$.  Extend this map from $\cA_i$ to $\RR$ with the same
Lipschitz constant while preserving its value at zero.  The coordinatewise
contraction property in Theorem~\ref{theorem:contraction-RC}, the symmetry
of the Rademacher variables, and
Assumption~\ref{assumption:Lipschitz-loss-Tn-f-Ln-ellf}(2) give
$\circled{2}$.  The fact that $(f-f_n^*)/2\in\cF$, because $\cF$ is
symmetric and convex, gives $\circled{3}$.  Finally,
Lemma~\ref{lemma:STLC-kernel-ind} gives
\bal\label{eq:STLC-kernel-seg1}
2\Expect{}{\sup_{h \colon h \in \Delta_{\cF},\tT_n(h) \le r}
R^{(\textup{ind})}_{\bsigma,\bY^{(u)}}h}
\le8L {\tilde \varphi_u}\pth{B_1r/4}.
\eal
Applying (\ref{eq:STLC-kernel-seg1}) for the two sample sizes gives
\bals
C_{L_0}\Expect{\bY^{(u)},\bsigma}{\sup_{h\in\Delta_{\cF}\colon\tT_n(h)\le r}
R^{(\textup{ind})}_{\bsigma,\bY^{(u)}}h}
&\le4C_{L_0}L\tilde\varphi_u\pth{rB_1/4}\defeq\varphi_u(r),\\
C_{L_0}\Expect{\bY^{(m)},\bsigma}{\sup_{h\in\Delta_{\cF}\colon\tT_n(h)\le r}
R^{(\textup{ind})}_{\bsigma,\bY^{(m)}}h}
&\le4C_{L_0}L\tilde\varphi_m\pth{rB_1/4}\defeq\varphi_m(r).
\eals
For each fixed $Q$, the functions appearing inside the minima in
(\ref{eq:varphi-STLC-kernel-u-ind}) and
(\ref{eq:varphi-STLC-kernel-m-ind}) are sub-root.  Pointwise minima and
positive sums preserve the sub-root property.  Thus
$\varphi_{u,m}(r)\defeq\varphi_u(r)+\varphi_m(r)$ is sub-root and, by
Corollary~\ref{corollary:STLC-delta-ell-f-ind-rc},
satisfies (\ref{eq:STLC-cond-psi-general}).  Let $r_{u,m}$ be its positive
fixed point.

We next bound this fixed point.  Put
$A\defeq4C_{L_0}L$ and $\beta\defeq B_1/4$.  For
$Q\in\set{0,\ldots,n}$, define
\bals
D_Q&\defeq\sqrt{\frac{\beta Q}{u}}+
\sqrt{\frac{\beta Q}{m}},\\
E_Q&\defeq\mu\pth{
\sqrt{\frac{\sum_{q=Q+1}^n\hat\lambda_q}{u}}+
\sqrt{\frac{\sum_{q=Q+1}^n\hat\lambda_q}{m}}
}.
\eals
The fixed-point identity and the fact that each minimum is bounded above
by its value at the same $Q$ give
$r_{u,m}=\varphi_{u,m}(r_{u,m})
\le A\pth{D_Q\sqrt{r_{u,m}}+E_Q}$.
The inequality $2st\le s^2+t^2$, which follows from Cauchy--Schwarz, gives
$AD_Q\sqrt{r_{u,m}}\le r_{u,m}/2+A^2D_Q^2/2$. Moreover,
$D_Q^2\le2\beta Q(1/u+1/m)$.  Consequently, there is a constant $c_K>0$,
depending only on $B',L_0,L$, and $\mu$, such that
$r_{u,m}\le A^2D_Q^2+2AE_Q\le c_Kr(u,m,Q)$ for
$Q\in\set{0,\ldots,n}$.
Therefore
\bal\label{eq:STLC-kernel-fixed-point-bound}
r_{u,m}\le c_K\min_{Q\in\set{0,\ldots,n}}r(u,m,Q).
\eal
We next construct a sub-root function $\psi^*_{u,m}$ satisfying
(\ref{eq:STLC-cond-um-delta-star-ell-f-psi-star}) in
Theorem~\ref{theorem:STLC-delta-star-ell-f}.  For every
$h=\ell_f-\ell_{f_n^*}\in\Delta^*_{\cF}$, the representation $(f,f_n^*)$
in (\ref{eq:tTn-def}) gives $\tT_n(h)\le2B_K\cL_n(h)$.  Therefore
$\set{h\in\Delta^*_{\cF}\colon B_K\cL_n(h)\le r}
\subseteq\set{h\in\Delta_{\cF}\colon\tT_n(h)\le2r}$.
The same inclusion holds after squaring the classes.  Define
$\psi^*_{u,m}\defeq\sqrt2\,\varphi_{u,m}$, and let $r^*$ be its positive
fixed point.  Since $\psi^*_{u,m}\ge\varphi_{u,m}$, fixed-point comparison
gives $r^*\ge r_{u,m}$.  Hence, for every $r\ge r^*$, condition
(\ref{eq:STLC-cond-psi-general}) may be applied at $2r$, and the sub-root
property gives
$\varphi_{u,m}(2r)\le\sqrt2\,\varphi_{u,m}(r)=\psi^*_{u,m}(r)$.
Together with the preceding inclusions, this proves
(\ref{eq:STLC-cond-um-delta-star-ell-f-psi-star}).  Finally,
$\psi^*_{u,m}(2r_{u,m})\le2\varphi_{u,m}(r_{u,m})=2r_{u,m}$,
so fixed-point comparison yields $r^*\le2r_{u,m}$.  Combining this with
(\ref{eq:STLC-kernel-fixed-point-bound}) gives
$r^*\le2r_{u,m}\le
2c_K\min_{Q\in\set{0,\ldots,n}}r(u,m,Q)$.

All hypotheses of
Theorem~\ref{theorem:STLC-delta-ell-f-excess-risk-upper-bound} are now
verified.  That theorem, (\ref{eq:STLC-kernel-fixed-point-bound}), and the
preceding bound imply, with probability at least $1-3\exp(-x)$,
\bals
\cE(\hat f_m)
&\le c_1r_{u,m}+\frac{4B_Kc_{\Delta}}{K_0}r^*
+\frac{c_3x}{N_{u,m}}\\
&\le c_K\pth{c_1+\frac{8B_Kc_{\Delta}}{K_0}}
\min_{Q\in\set{0,\ldots,n}}r(u,m,Q)
+\frac{c_3x}{N_{u,m}}.
\eals
Taking
$c_5\ge\max\set{c_K(c_1+8B_Kc_{\Delta}/K_0),c_3}$ proves
(\ref{eq:STLC-kernel-excess-loss}).  The dependence asserted for $c_5$
follows from the dependence of $c_1,c_{\Delta},c_3$, and $c_K$.
\end{proof}

\section{More Detailed Comparisons to Existing Works}
\label{sec:more-comparison}
We present more detailed comparisons of our results to existing works in this section.

\subsection{Assumption provenance}
Table~\ref{tab:assumption-provenance} separates the assumptions of the three
statistical settings from the proof devices introduced here.  Each comparison
is made under the same corresponding assumptions as the cited predecessor.
\begin{table}[H]
\caption{Provenance of the assumptions.  ``Same'' means that the cited prior
local-complexity result uses the listed assumptions in the corresponding
setting.}
\label{tab:assumption-provenance}
\centering
\small
\setlength{\tabcolsep}{4pt}
\renewcommand{\arraystretch}{1.1}
\begin{tabular}{@{}>{\raggedright\arraybackslash}p{0.16\linewidth}>{\raggedright\arraybackslash}p{0.50\linewidth}>{\raggedright\arraybackslash}p{0.24\linewidth}@{}}
\hline
Setting & Assumptions used here & Relation to prior work \\
\hline
Generic excess risk & Bounded loss, existence of $f_n^*$,
$T_n(h)\le B\cL_n(h)$ on $\Delta^*_{\cF}$, and existence of the splitwise
empirical minimizers & Same as the corresponding analyses in
\citep{yang2025a-concentration-sampling-without-replacement,
yang2026improvedgeneralizationboundstransductive} \\
Realizable VC & Binary class of finite VC dimension, realizability, and
$u\ge m\ge\dVC$ & Same as
\citep[Theorem~6.1]{yang2026improvedgeneralizationboundstransductive} \\
Kernel & RKHS ball, bounded Lipschitz loss, existence of $f_n^*$, and the
prediction-localization condition in
Assumption~\ref{assumption:Lipschitz-loss-Tn-f-Ln-ellf}(2) & Same as the
localized kernel analyses compared in
Section~\ref{sec:more-comparison-TKL}. Same structural prediction-localization
assumptions as the BBM kernel analysis invoked by
\citep[Corollaries 14 and 15]{TolstikhinBK14-local-complexity-TRC}. \\
\hline
\end{tabular}
\end{table}
The swap-walk concentration argument, the surrogate functional, and the
peeling construction are proof devices.  They do not impose an additional
distributional, realizability, spectral, or sample-imbalance assumption.

\subsection{More Detailed Comparison for Transductive Learning Over Binary-Valued Function Classes}
\label{sec:more-comparison-finite-VC-dim}

A related result for binary-valued classes of finite VC-dimension is
\citep[Theorem 6.1]{yang2026improvedgeneralizationboundstransductive}.  Under
the conditions of
Theorem~\ref{theorem:STLC-delta-ell-f-excess-risk-upper-bound-VC-dim}, it
states that there are absolute positive constants $\bar c_0,\bar c_1$ such
that, for every $x>0$ and $\delta\in(0,1)$, with probability at least
$1-\exp(-x)-\delta$,
$\cL_u(\ell_{\hat f_m})\le\bar c_0\dVC\log(me/\dVC)/m
+\bar c_1\log_2(4m/\delta)x/m$.
Theorem~\ref{theorem:STLC-delta-ell-f-excess-risk-upper-bound-VC-dim}
removes the factor $\log_2(4m/\delta)$ from the confidence term.  Here
$N_{u,m}=m$ because this comparison assumes $u\ge m$.
The earlier result \citep[Theorem 6.1]{yang2026improvedgeneralizationboundstransductive} already achieved the leading
$\dVC\log(me/\dVC)/m$ dependence, replacing the older $\log n$ factor by
$\log m$.  The precise improvement in Theorem~\ref{theorem:STLC-delta-ell-f-excess-risk-upper-bound-VC-dim} is therefore the
removal of the auxiliary probability $\delta$ and its logarithmic confidence
penalty.  This leading term agrees with the classical LRC-based upper bound in
\citep[Corollary~3.7]{bartlett2005}, but not with the optimal realizable
inductive PAC rate of order $\dVC/m$.  Examples of optimal PAC constructions
include \citet{hanneke2016optimal},
\citet{adenali2023optimal}, and
\citet{adenali2024majority}.  The transductive lower bound also has order $\dVC/m$~\citep{tolstikhin2016minimaxlowerboundsrealizable}.
\citet{yang2026improvedgeneralizationboundstransductive} also attains a
log-free randomized rate of $\cO(\dVC/m)$.

\subsection{More Detailed Comparisons for Transductive Kernel Learning}
\label{sec:more-comparison-TKL}
We compare (\ref{eq:STLC-kernel-excess-loss}) with
\citep[Corollaries 14 and 15]{TolstikhinBK14-local-complexity-TRC}.  Under the
unit-RKHS-ball and unit-kernel-diagonal normalization in that work, the two
corollaries imply that, for every $x>0$, with probability at least
$1-2\exp(-x)$,
\bal\label{eq:local-complexity-excess-risk-TKL}
\cE(\hat f_{ m}) \le \cO\pth{\frac nu r_m^*+ \frac nm r_u^*+x\pth{\frac 1m +\frac 1u}},
\eal
where $r_u^*$ and $r_m^*$ are the localized fixed points in those corollaries.
For a constant $C_L$ depending only on the loss Lipschitz constant, they obey
$r_s^*\le C_L\min_{Q\in\set{0,\ldots,s}}
\pth{Q/s+\sqrt{\sum_{q=Q+1}^{n}\hat\lambda_q/s}}$ for $s\in\set{u,m}$.
The comparison is direct at the level of the displayed bounds.
Equation~(\ref{eq:local-complexity-excess-risk-TKL}) multiplies its
two fixed-point bounds by $n/u$ and $n/m$, whereas
(\ref{eq:STLC-kernel-excess-loss}) contains no such prefactors.  No universal
asymptotic conclusion follows from this algebra alone because the empirical
spectrum may change with the full sample.  Independently of spectral decay,
choosing $Q=0$ and using
$\sum_{q=1}^n\hat\lambda_q=\tr{\bK/n}$ gives the deterministic estimate
$\min_Q r(u,m,Q)\le
\sqrt{\tr{\bK/n}}\pth{u^{-1/2}+m^{-1/2}}$.

Furthermore, \citep[Theorem 6.2]{yang2026improvedgeneralizationboundstransductive}
and \citep[Theorem 4.1]{yang2025a-concentration-sampling-without-replacement}
give related TKL bounds.  Under the assumptions of the former result, there is
a constant $\bar c_5>0$ with the parameter dependence stated there.  For every
$x>0$ and $\delta\in(0,1/3)$, with probability at least
$1-3\exp(-x)-3\delta$,
$\cE(\hat f_m)\le\bar c_5\pth{\min_{Q\in\set{0,\ldots,n}}r(u,m,Q)
+\log_2(4N_{u,m}/\delta)x/N_{u,m}}$.
Theorem~\ref{theorem:STLC-kernel} removes from this bound the factor
$\log_2(4\min\set{u,m}/\delta)$.  A bound of the same broad form is available
under the hypotheses of
\citep[Theorem 4.1]{yang2025a-concentration-sampling-without-replacement}, but
that result requires the imbalanced regime $u\gg m^2$ or $m\gg u^2$.

\subsection{Understanding Theorem~\ref{theorem:STLC-kernel} by Kernel Regimes}
\label{sec:limitations}
We provide analysis about the informative regime for the excess risk bound for TKL in Theorem~\ref{theorem:STLC-kernel}.
If $\rank(\bK)\le R$ for an integer $R\in\set{0,\ldots,n}$, choosing $Q=R$ gives
$r(u,m,R)=R(1/u+1/m)$.  If
$\hat\lambda_q\lesssim q^{-2\alpha}$ for $\alpha>1/2$, choosing
$Q\asymp N_{u,m}^{1/(2\alpha+1)}$ gives
$\min_Qr(u,m,Q)\lesssim
N_{u,m}^{-2\alpha/(2\alpha+1)}$.  Without favorable spectral decay, $Q=0$
still yields order $u^{-1/2}+m^{-1/2}$ when
$\tr{\bK/n}$ is uniformly bounded.

\section{More Discussions}
\label{sec:more-discussion}
Concentration inequalities for sampling without replacement have been studied
extensively~\citep{BARDENET2015-sampling-without-replacement,
Tolstikhin2017-sampling-without-replacement}, including modified
log-Sobolev results on the multislice~\citep{Sambale2022}.  As detailed in
Section~\ref{sec:concentration-general-sup-empirical},
Corollary~\ref{corollary:concentration-gu} avoids the two regime-dependent
terms in the bounds of \citep{TolstikhinBK14-local-complexity-TRC}.  The
localized analysis in
\citep{yang2025a-concentration-sampling-without-replacement} applies only in
the imbalanced regimes $m\gg u^2$ or $u\gg m^2$.  The present results require
no such relation between $m$ and $u$.  They also give a realizable VC bound
that is minimax-optimal up to $\log(me/\dVC)$ for $m\ge9$ and a
spectrum-dependent kernel bound.  The open directions are collected in
Section~\ref{sec:limitations}.

\subsection*{Limitations and future work}
The concentration theorem requires uniformly bounded functions.  Extending
the swap-walk entropy method to sub-Gaussian, sub-exponential, heavy-tailed, or
robust-loss settings is open.  Realizability is used only by the finite-VC
application, which does not cover agnostic classification. The generic constants are not optimized and
should be interpreted at the rate level.  The kernel expression is most informative under
low effective dimension.  Future works include graph-based
transduction, semi-supervised model selection, and other finite-population
empirical processes.  Appendix~\ref{sec:limitations} gives quantitative
kernel examples and further detail.

\end{document}